\documentclass[10pt,journal,cspaper,compsoc]{IEEEtran}
\ifCLASSOPTIONcompsoc
  \usepackage[nocompress]{cite}
\else
  \usepackage{cite}
\fi

\newcommand{\symd}{\operatorname{Sym}^+_d}
\newcommand{\grsmn}{\mathcal{G}^r_n}
\newcommand{\mnd}{\mathcal{M}}

\newcommand{\Rn}{\mathbb{R}^n}
\newcommand{\X}{\mathcal{X}}
\newcommand{\Hg}{\mathcal{H}_g}

\renewcommand{\vec}[1]{\mathbf{#1}}

\usepackage{caption}
\usepackage{subcaption}
\usepackage{tabularx}
\usepackage{multirow}
\usepackage{amsmath}
\usepackage{amssymb}
\usepackage{graphicx}

\usepackage{amsthm}           
\newtheorem{thm}{Theorem}[section]
\newtheorem{lem}[thm]{Lemma}

\newtheorem{cor}[thm]{Corollary}

\newtheorem{mydef}[thm]{Definition}

\begin{document}
\title{\vspace{-0.35cm}Kernel Methods on Riemannian Manifolds with Gaussian RBF Kernels}
%
%
%
%
\author{\vspace{-0.22cm}Sadeep~Jayasumana,~\IEEEmembership{Student Member,~IEEE,}
        Richard~Hartley,~\IEEEmembership{Fellow,~IEEE,}
        Mathieu~Salzmann,~\IEEEmembership{Member,~IEEE,}
        Hongdong~Li,~\IEEEmembership{Member,~IEEE,}
        and~Mehrtash~Harandi,~\IEEEmembership{Member,~IEEE}\vspace{-0.2cm}
\IEEEcompsocitemizethanks{\IEEEcompsocthanksitem The authors are with the College of Engineering and Computer Science, Australian National University, Canberra and NICTA, Canberra.\protect\\
Contact E-mail: sadeep.jayasumana@anu.edu.au
\IEEEcompsocthanksitem NICTA is funded by the Australian Government as represented by the Department of Broadband, Communications and the Digital Economy and the Australian Research Council (ARC) through the ICT Centre of Excellence program.}%
}

\IEEEcompsoctitleabstractindextext{%
\vspace{-0.25cm}
\begin{abstract}
In this paper, we develop an approach to exploiting kernel methods with manifold-valued data. In many computer vision problems, the data can be naturally represented as points on a Riemannian manifold. Due to the non-Euclidean geometry of Riemannian manifolds, usual Euclidean computer vision and machine learning algorithms yield inferior results on such data. In this paper, we define Gaussian radial basis function (RBF)-based positive definite kernels on manifolds that permit us to embed a given manifold with a corresponding metric in a high dimensional reproducing kernel Hilbert space. These kernels make it possible to utilize algorithms developed for linear spaces on nonlinear manifold-valued data.
Since the Gaussian RBF defined with any given metric is not always positive definite, we present a unified framework for analyzing the positive definiteness of the Gaussian RBF on a generic metric space. We then use the proposed framework to identify positive definite kernels on two specific manifolds commonly encountered in computer vision: the Riemannian manifold of symmetric positive definite matrices and the Grassmann manifold, i.e., the Riemannian manifold of linear subspaces of a Euclidean space. We show that many popular algorithms designed for Euclidean spaces, such as support vector machines, discriminant analysis and principal component analysis can be generalized to Riemannian manifolds with the help of such positive definite Gaussian kernels.

\end{abstract}

\vspace{-0.2cm}
\vspace{-0.1cm}}

\maketitle\vspace{-2cm}

\IEEEdisplaynotcompsoctitleabstractindextext

%

%

\ifCLASSOPTIONcompsoc
  \noindent\raisebox{2\baselineskip}[0pt][0pt]%
  {\parbox{\columnwidth}{\section{Introduction}\label{sec:introduction}%
  \global\everypar=\everypar}}%
  \vspace{-1\baselineskip}\vspace{-\parskip}\par
\else
  \section{Introduction}\label{sec:introduction}\par
\fi
%

%
%
%
%
\IEEEPARstart{M}{athematical} entities that do not form Euclidean spaces but lie on nonlinear manifolds are often encountered in computer vision. Examples include 3D rotation matrices that form the Lie group $SO(3)$, normalized histograms that form the unit $n$-sphere $S^n$, symmetric positive definite (SPD) matrices and linear subspaces of a Euclidean space. Recently, the latter two manifolds have drawn significant attention in the computer vision community due to their widespread applications. SPD matrices, which form a Riemannian manifold when endowed with an appropriate metric~\cite{Pennec06}, are encountered in computer vision in the forms of covariance region descriptors~\cite{Tuzel06, Tuzel08}, diffusion tensors~\cite{Pennec06, Li2012_L2ECM} and structure tensors~\cite{Goh08, CaseiroICCV2011}. Linear subspaces of a Euclidean space, known to form a Riemannian manifold named the Grassmann manifold, are commonly used to model image sets~\cite{Hamm08, Harandi2011} and videos~\cite{Turaga11}.

Since manifolds lack a vector space structure and other Euclidean structures such as norm and inner product, many popular computer vision and machine learning algorithms including support vector machines (SVM), principal component analysis (PCA) and mean-shift clustering cannot be applied in their original forms on manifolds. One way of dealing with this difficulty is to neglect the nonlinear geometry of manifold-valued data and apply Euclidean methods directly. As intuition suggests, this approach often yields poor accuracy and undesirable effects, see, e.g.,~\cite{Arsigny06, Pennec06} for SPD matrices.

When the manifold under consideration is Riemannian, another common approach used to cope with its nonlinearity consists in approximating the manifold-valued data with its projection to a tangent space at a particular point on the manifold, for example, the mean of the data. Such tangent space approximations are often calculated successively as the algorithm proceeds~\cite{Tuzel08}. However, mapping data to a tangent space only yields a first-order approximation of the data that can be distorted, especially in regions far from the origin of the tangent space. Moreover, iteratively mapping back and forth to the tangent spaces significantly increases the computational cost of the algorithm. It is also difficult to choose the origin of the tangent space, which heavily affects the accuracy of this approximation.

In Euclidean spaces, the success of many computer vision algorithms arises from the use of kernel methods~\cite{Shawe-Taylor2004book, Scholkopf2002book}. Therefore, one could think of following the idea of kernel methods in $\Rn$ and embed a manifold in a high dimensional Reproducing Kernel Hilbert Space (RKHS), where linear geometry applies. Many Euclidean algorithms can be directly generalized to an RKHS, which is a vector space that possesses an important structure: the inner product. Such an embedding, however, requires a kernel function defined on the manifold, which, according to Mercer's theorem~\cite{Scholkopf2002book}, should be positive definite. While many positive definite kernels are known for Euclidean spaces, such knowledge remains limited for manifold-valued data.

The Gaussian radial basis function (RBF) kernel $\exp(-\gamma\|\vec{x} - \vec{y}\|^2)$ is perhaps the most popular and versatile kernel in Euclidean spaces. It would therefore seem attractive to generalize this kernel to manifolds by replacing the Euclidean distance in the RBF by a more accurate nonlinear distance measure on the manifold, such as the geodesic distance. However, a kernel formed in this manner is \emph{not} positive definite in general.

In this paper, we aim to generalize the successful and powerful kernel methods to manifold-valued data.  To this end, we analyze the Gaussian RBF kernel on a generic manifold and provide necessary and sufficient conditions for the Gaussian RBF kernel generated by a distance function on any nonlinear manifold to be positive definite. This lets us generalize kernel methods to manifold-valued data while preserving the favorable properties of the original algorithms. 

We apply our theory to analyze the positive definiteness of Gaussian kernels defined on two specific manifolds: the Riemannian manifold of SPD matrices and the Grassmann manifold. Given the resulting positive definite kernels, we discuss different kernel methods on these two manifolds including, kernel SVM, multiple kernel learning (MKL) and kernel PCA. Our experiments on a variety of computer vision tasks, such as pedestrian detection, segmentation, face recognition and action recognition, show that our manifold kernel methods outperform the corresponding Euclidean algorithms that neglect the manifold geometry, as well as other state-of-the-art techniques specifically designed for manifolds.
\vspace{-0.2cm}
\section{Related Work}
In this paper, we focus on the Riemannian manifold of SPD matrices and on the Grassmann manifold. SPD matrices find a variety of applications in computer vision~\cite{Carreira_ECCV2012_SOPooling}. For instance, covariance region descriptors are used in object detection~\cite{Tuzel08}, texture classification~\cite{Tuzel06, Li_ICCV2013_LE}, object tracking, action recognition and human recognition~\cite{Harandi12, Tosato2013PAMI}. Diffusion Tensor Imaging (DTI) was one of the pioneering fields for the development of non-linear algorithms on SPD matrices~\cite{Pennec06, Arsigny06}. In optical flow estimation and motion segmentation, structure tensors are often employed to encode important image features, such as texture and motion~\cite{Goh08, CaseiroICCV2011}. Structure tensors have also been used in single image segmentation~\cite{Malcolm07}.

Grassmann manifolds are widely used to encode image sets and videos for face recognition~\cite{Hamm08, Harandi2011}, activity recognition~\cite{Turaga11, Lui2012}, and motion grouping~\cite{Lui2012}. In image set based face recognition, a set of face images of the same person is represented as a linear subspace, hence a point on a Grassmann manifold. In activity recognition and motion grouping, a subspace is formed either directly from the sequence of images containing a specific action, or from the parameters of a dynamic model obtained from the sequence~\cite{Turaga11}.

In recent years, several optimization algorithms have been proposed for Riemannian manifolds. In particular, LogitBoost for classification on Riemannian manifolds was introduced in \cite{Tuzel08}. This algorithm has the drawbacks of approximating the manifold by tangent spaces and not scaling with the number of training samples due to the heavy use of exponential/logarithmic maps to transit between the manifold and the tangent space, as well as of gradient descent based Karcher mean calculation. Here, our positive definite kernels enable us to use more efficient and accurate classification algorithms on manifolds without requiring tangent space approximations. Furthermore, as shown in \cite{Tosato10, CaseiroCVPR2013_Rolling}, extending existing kernel-free, manifold-based binary classifiers to the multi-class case is not straightforward. In contrast, the kernel-based classifiers on manifolds described in this paper can readily be used in multi-class scenarios.

In~\cite{Goh08}, dimensionality reduction and clustering methods were extended to manifolds by designing Riemannian versions of Laplacian Eigenmaps (LE), Locally Linear Embedding (LLE) and Hessian LLE (HLLE). Clustering was performed after mapping to a low dimensional space which does not necessarily preserve all the information in the original data. Instead, we use our kernels to perform clustering in a higher dimensional RKHS that embeds the manifold of interest.

The use of kernels on SPD matrices has previously been advocated for locality preserving projections~\cite{Harandi12WACV} and sparse coding~\cite{Harandi12}. In the first case, the kernel, derived from the affine-invariant distance, is not positive definite in general~\cite{Harandi12WACV}. In the second case, the kernel is positive definite only for some values of the Gaussian bandwidth parameter $\gamma$~\cite{Harandi12}. For all kernel methods, the optimal choice of $\gamma$ largely depends on the data distribution and hence constraints on $\gamma$ are not desirable. Moreover, many popular automatic model selection methods require $\gamma$ to be continuously variable~\cite{Chapelle02}.

In \cite{Hamm08, Hamm08_nips}, the Projection kernel and its extensions were introduced and employed for classification on Grassmann manifolds. While those kernels are analogous to the linear kernel in Euclidean spaces, our kernels are analogous to the Gaussian RBF kernel. In Euclidean spaces, the Gaussian kernel has proven more powerful and versatile than the linear kernel. As shown in our experiments, this also holds for kernels on manifolds.

Recently, mean-shift clustering with the heat kernel on Riemannian manifolds was introduced~\cite{Caseiro12}. However, due to the mathematical complexity of the kernel function, computing the exact kernel is not tractable and hence only an approximation of the true kernel was used. Parallel to our work, kernels on SPD matrices and on Grassmann manifolds were used in~\cite{Vemulapalli2013}, albeit without explicit proof of their positive definiteness. In contrast, in this paper, we introduce a unified framework for analyzing the positive definiteness of the Gaussian kernel defined on any manifold and use this framework to identify provably positive definite kernels on the manifold of SPD matrices and on the Grassmann manifold.

Other than for satisfying Mercer's theorem to generate a valid RKHS, positive definiteness of the kernel is a required condition for the convergence of many kernel based algorithms. For instance, the Support Vector Machine (SVM) learning problem is convex only when the kernel is positive definite~\cite {Platt99}. Similarly, positive definiteness of all participating kernels is required to guarantee the convexity in Multiple Kernel Learning (MKL)~\cite{Varma07}. Although theories have been proposed to exploit non-positive definite kernels~\cite{Ong04, Wu05ananalysis}, they have not experienced widespread success. Many of these methods first enforce positive definiteness of the kernel matrix by flipping or shifting its negative eigenvalues~\cite{Wu05ananalysis}. As a consequence, they result in a distortion of information and become inapplicable with large size kernels, which are not uncommon in learning problems.

It is important to note the difference between this work and manifold-learning methods such as~\cite{BelkinNiyogi_ML_04}. We work with data sampled from a manifold whose geometry is well known. In contrast, manifold-learning methods attempt to learn the structure of an underlying unknown manifold from data samples. Furthermore, those methods often assume that noise-free data samples lie on a manifold from which noise push them away. In our study, data points, regardless of their noise content, always lie on the mathematically well-defined manifold.

\vspace{-0.1cm}
\section{Manifolds in Computer Vision}

A \emph{topological manifold}, generally referred to as simply a \emph{manifold}, is a topological space that is locally homeomorphic to the $n$-dimensional Euclidean space $\Rn$, for some $n$. Here, $n$ is referred to as the dimensionality of the manifold. A \emph{differentiable manifold} is a topological manifold that has a globally defined differential structure. The tangent space at a given point on a differentiable manifold is a vector space that consists of the tangent vectors of all possible curves passing through the point.

A \emph{Riemannian manifold} is a differentiable manifold equipped with a smoothly varying inner product on each tangent space. The family of inner products on all tangent spaces is known as the \emph{Riemannian metric} of the manifold. It enables us to define various geometric notions on the manifold such as the angle between two curves and the length of a curve. The \textit{geodesic distance} between two points on the manifold is defined as the length of the shortest curve connecting the two points. Such shortest curves are known as \textit{geodesics} and are analogous to straight lines in $\mathbb{R}^n$.

The geodesic distance induced by the Riemannian metric is the most natural measure of dissimilarity between two points lying on a Riemannian manifold. However, in practice, many other nonlinear distances or \emph{metrics} which do not necessarily arise from Riemannian metrics can also be useful for measuring dissimilarity on manifolds. It is worth noting that the term \emph{Riemannian metric} refers to a family of inner products while the term \emph{metric} refers to a distance function that satisfies the four metric axioms. A nonempty set endowed with a metric is known as a \emph{metric space} which is a more abstract space than a Riemannian manifold.

In the following, we discuss two important Riemannian manifolds commonly found in computer vision.

\vspace{-0.2cm}
\subsection{The Riemannian Manifold of SPD Matrices}
A $d \times d$, real Symmetric Positive Definite (SPD) matrix $S$ has the property: $\vec{x}^TS\vec{x} > 0$ for all nonzero $\vec{x} \in \mathbb{R}^d$. The space of $d \times d$ SPD matrices, which we denote by $\symd$, is clearly not a vector space since an SPD matrix when multiplied by a negative scalar is no longer SPD. Instead, $\symd$ forms a convex cone in the $d^2$-dimensional Euclidean space.

The geometry of the space $\symd$ is best explained with a Riemannian metric which induces an infinite distance between an SPD matrix and a non-SPD matrix~\cite{Pennec06, Arsign05TechReport}. Therefore, the geodesic distance induced by such a Riemannian metric is a more accurate distance measure on $\symd$ than the Euclidean distance in the $d^2$-dimensional Euclidean space it is embedded in. Two popular Riemannian metrics proposed on $\symd$ are the affine-invariant Riemannian metric~\cite{Pennec06} and the log-Euclidean Riemannian metric~\cite{Arsign05TechReport, Arsigny06}. They result in the affine-invariant geodesic distance and the log-Euclidean geodesic distance, respectively. These two distances are so far the most widely used metrics on $\symd$.

Apart from these two geodesic distances, a number of other metrics have been proposed for $\symd$ to capture its nonlinearity~\cite{Sra12}. For a detailed review of metrics on $\symd$, we refer the reader to~\cite{Dryden09thestatistical}.

\vspace{-0.2cm}
\subsection{The Grassmann Manifold}
A point on the $(n, r)$ Grassmann manifold, where $n > r$, is an $r$-dimensional linear subspace of the $n$-dimensional Euclidean space. Such a point is generally represented by an $n \times r$ matrix $Y$ whose columns store an orthonormal basis of the subspace. With this representation, the point on the Grassmann manifold is the subspace spanned by the columns of $Y$ or $\operatorname{span}(Y)$ which we denote by $[Y]$. We use $\grsmn$ to denote the $(n, r)$ Grassmann manifold.

The set of $n \times r$ $(n > r)$ matrices with orthonormal columns forms a
manifold known as the $(n, r)$ \emph{Stiefel manifold}. From its embedding
in the $nr$-dimensional Euclidean space, the $(n, r)$ Stiefel manifold
inherits a canonical Riemannian metric~\cite{Edelman98}. 
Points on
$\grsmn$ are equivalence classes of $n \times r$ matrices with orthonormal
columns, where two matrices are equivalent if their columns span the same
$r$-dimensional subspace. Thus, the orthogonal group ${\cal O}_r$ acts via
isometries (change of orthogonal basis)
on the Stiefel manifold by multiplication on the right, and $\grsmn$ can be identified
as the set of orbits of this action.  Since this action is both
free and proper,  $\grsmn$ forms a manifold, and it is given a Riemannian 
structure by equipping it with
the standard {\em normal} Riemannian metric derived from the metric
on the Stiefel manifold.

A geodesic distance on the Grassmann manifold, called the arc length distance, can be derived from its canonical geometry described above. The arc length distance between two points on the Grassmann manifold turns out to be the $l_2$ norm of the vector formed by the principal angles between the two subspaces. Several other metrics on this manifold can be derived from the principal angles. We refer the reader to \cite{Edelman98} for more details and properties of different Grassmann metrics.

\section{Hilbert Space Embedding of Manifolds}

An \emph{inner product space} is a vector space equipped with an inner product. A \emph{Hilbert space} is an (often infinite-dimensional) inner product space which is complete with respect to the norm induced by the inner product. A Reproducing Kernel Hilbert Space (RKHS) is a special kind of Hilbert space of functions on some nonempty set $\X$ in which all evaluation functionals are bounded and hence continuous~\cite{Aronszajn1950}. The inner product of an RKHS of functions on $\X$ can be defined by a bivariate function on $\X \times \X$, known as the \emph{reproducing kernel} of the RKHS. 

Many useful computer vision and machine learning algorithms developed for Euclidean spaces depend only on the notion of inner product, which allows us to measure angles and also distances. Therefore, such algorithms can be extended to Hilbert spaces without effort. A notable special case arises with RKHSs where the inner product of the Hilbert space can be evaluated using a kernel function without computing the actual vectors. This concept, known as the \emph{kernel trick}, is commonly utilized in machine learning in the following setting: input data in some $n$-dimensional Euclidean space $\Rn$ are mapped to a high dimensional RKHS where some learning algorithm, which requires only the inner product, is applied. We never need to calculate actual vectors in the RKHS since the learning algorithm only requires the inner product of the RKHS, which can be calculated by means of a kernel function defined on $\Rn \times \Rn$. A variety of algorithms can be used with the kernel trick, such as support vector machines (SVM), principal component analysis (PCA), Fisher discriminant analysis (FDA), $k$-means clustering and ridge regression.

Embedding lower dimensional data in a higher dimensional RKHS is commonly employed with data that lies in a Euclidean space. The theoretical concepts of such embeddings can directly be extended to manifolds. Points on a manifold $\mathcal{M}$ are mapped to elements in a high (possibly infinite) dimensional Hilbert space $\mathcal{H}$, a subspace of the space spanned by real-valued functions\footnote{We limit the discussion to real Hilbert spaces and real-valued kernels, since they are the most useful kind in learning algorithms. However, the theory holds for complex Hilbert spaces and complex-valued kernels as well.} on $\mathcal{M}$. A kernel function $k: (\mathcal{M} \times \mathcal{M}) \to \mathbb{R}$ is used to define the inner product on $\mathcal{H}$, thus making it an RKHS. The technical difficulty in utilizing Hilbert space embeddings with manifold-valued data arises from the fact that, according to Mercer's theorem, the kernel must be positive definite to define a valid RKHS. While many positive definite kernel functions are known for ${\mathbb{R}}^n$, generalizing them to manifolds is not straightforward.

Identifying such positive definite kernel functions on manifolds would, however, be greatly beneficial. Indeed, embedding a manifold in an RKHS has two major advantages: First, the mapping transforms the nonlinear manifold into a (linear) Hilbert space, thus making it possible to utilize algorithms designed for linear spaces with manifold-valued data. Second, as evidenced by the theory of kernel methods in Euclidean spaces, it yields a much richer high-dimensional representation of the original data, making tasks such as classification easier.

In the following we build a framework that enables us to define positive definite kernels on manifolds.

\vspace{-0.3cm}
\section{Theory of Positive and Negative Definite Kernels}
In this section, we present some general results on positive and negative definite kernels. These results will be useful for our derivations in later sections. We start with the definition of real-valued positive and negative definite kernels on a set~\cite{Berg84}. Note that by the term \emph{kernel} we mean a real-valued bivariate function hereafter.

\begin{mydef}
\label{def:posdef}
Let $\mathcal{X}$ be a nonempty set. A kernel $f: (\mathcal{X} \times \mathcal{X}) \to \mathbb{R}$ is called \textbf{positive definite} if it is symmetric (i.e., $f(x, y) = f(y, x)$ for all $x, y \in \mathcal{X}$) and
\vspace{-0.1cm}
$$\sum_{i,j = 1}^{m}c_ic_jf(x_i, x_j) \ge 0$$
for all $m \in \mathbb{N}, \{x_1,\ldots,x_m\} \subseteq \mathcal{X}$ and $\{c_1,..., c_m\} \subseteq \mathbb{R}$. The kernel $f$ is called \textbf{negative definite} if it is symmetric and 
$$\sum_{i,j = 1}^{m}c_ic_jf(x_i, x_j) \le 0$$
for all $m \in \mathbb{N}, \{x_1,\ldots,x_m\} \subseteq \mathcal{X}$ and $\{c_1,..., c_m\} \subseteq \mathbb{R}$ with $\sum_{i = 1}^{m}c_i = 0$.
\end{mydef}

It is important to note the additional constraint on $\sum c_i$ for the negative definite case. Due to this constraint, some authors refer to this latter kind as \emph{conditionally negative definite}. However, in this paper, we stick to the most common terminology used in the literature. 

We next present the following theorem which plays a central role in this paper. It was introduced by Schoenberg in 1938~\cite{Schoenberg1938}, well before the theory of Reproducing Kernel Hilbert Spaces was established in 1950~\cite{Aronszajn1950}.

\begin{thm} \label{thm:sho}
Let $\mathcal{X}$ be a nonempty set and $f : (\mathcal{X} \times \mathcal{X}) \to \mathbb{R}$ be a kernel. The kernel $\exp(-\gamma\,f(x,y))$ is positive definite for all $\gamma > 0$ if and only if $f$ is negative definite.
\end{thm}
\begin{proof}
We refer the reader to Theorem 3.2.2 of \cite{Berg84} for a detailed proof of this theorem.
\end{proof}

Note that this theorem describes Gaussian RBF--like exponential kernels that are positive definite for all $\gamma > 0$. One might also be interested in exponential kernels that are positive definite for only some values of $\gamma$. Such kernels, for instance, were exploited in \cite{Harandi12}. To our knowledge, there is no general result characterizing this kind of kernels. However, the following result can be obtained from the above theorem.

\begin{thm} \label{thm:new}
Let $\mathcal{X}$ be a nonempty set and $f : (\mathcal{X} \times \mathcal{X}) \to \mathbb{R}$ be a kernel. If the kernel $\exp(-\gamma\,f(x,y))$ is positive definite for all $\gamma \in (0, \delta)$ for some $\delta > 0$, then it is positive definite for all $\gamma > 0$.
\end{thm}
\begin{proof}
If $\exp(-\gamma\,f(x,y))$ is positive definite for all $\gamma \in (0, \delta)$, it directly follows from Definition~\ref{def:posdef} that $1 - \exp(-\gamma\,f(x,y))$ is negative definite for all $\gamma \in (0, \delta)$. Therefore, the pointwise limit
$$
\lim_{\gamma \to 0^+} \frac{1 - \exp(-\gamma\,f(x,y))}{\gamma} = f(x, y)
$$
is also negative definite. Now, since $f(x, y)$ is negative definite, it follows from Theorem~\ref{thm:sho} that $\exp(-\gamma\,f(x,y))$ is positive definite for all $\gamma > 0$.
\end{proof}

Next, we highlight an interesting property of negative definite kernels. It is well known that a positive definite kernel represents the inner product of an RKHS~\cite{Aronszajn1950}. Similarly, a negative definite kernel represents the squared norm of a Hilbert space under some conditions stated by the following theorem.
\begin{thm}
\label{thm:negDefEmbedding}
Let $\mathcal{X}$ be a nonempty set and $f(x,y) : (\mathcal{X} \times \mathcal{X}) \to \mathbb{R}$ be a negative definite kernel. Then, there exists a Hilbert space $\mathcal{H}$ and a mapping $\psi : \mathcal{X} \to \mathcal{H}$ such that,
$$
f(x, y) = \|\psi(x) - \psi(y)\|^2 + h(x) + h(y)
$$
where $h:\mathcal{X} \to \mathbb{R}$ is a function which is nonnegative whenever $f$ is. Furthermore, if $f(x, x) = 0$ for all $x \in \mathcal{X}$, then $h = 0$. 
\end{thm}
\begin{proof}
The proof for a more general version of this theorem can be found in Proposition 3.3.2 of~\cite{Berg84}.
\end{proof}

Now, we state and prove a lemma that will be useful for the proof of our main theorem.
\begin{lem} \label{thm:innerprod}
Let $\mathcal{X}$ be a nonempty set, $\mathcal{V}$ be an inner product space, and $\psi : \mathcal{X} \to \mathcal{V}$ be a function. Then, $f : (\mathcal{X} \times \mathcal{X}) \to \mathbb{R}$ defined by $f(x, y) := {\| \psi(x) - \psi(y) \|}^2_{\mathcal{V}}$ is negative definite.
\end{lem}
\begin{proof}\belowdisplayskip=-12pt
\allowdisplaybreaks
The kernel $f$ is obviously symmetric. Based on Definition~\ref{def:posdef}, we then need to prove that $ \sum_{i,j = 1}^{m}c_ic_jf(x_i, x_j) \le 0 $ for all $m \in \mathbb{N}$, $ \{x_1,\ldots,x_m\} \subseteq \mathcal{X}$ and $\{c_1,..., c_m\} \subseteq \mathbb{R}$ with $\sum_{i = 1}^{m}c_i = 0$. Now,
\begin{align*}
&\sum_{i,j = 1}^{m}c_ic_jf(x_i, x_j) = \sum_{i,j = 1}^{m}c_ic_j{\Big\| \psi(x_i) - \psi(x_j)\Big\|}^2_{\mathcal{V}}\\
&=\sum_{i,j = 1}^{m}c_ic_j{\Big\langle\psi(x_i) - \psi(x_j), \psi(x_i) - \psi(x_j)\Big\rangle}_{\mathcal{V}}&\\
&=\sum_{j = 1}^{m}c_j\sum_{i = 1}^{m}c_i\Big\langle\psi(x_i), \psi(x_i)\Big\rangle_{\mathcal{V}}&\\
&\qquad -2\sum_{i,j = 1}^{m}c_ic_j\Big\langle\psi(x_i), \psi(x_j)\Big\rangle_{\mathcal{V}}\;\\
&\qquad +\sum_{i = 1}^{m}c_i\sum_{j = 1}^{m}c_j\Big\langle\psi(x_j), \psi(x_j)\Big\rangle_{\mathcal{V}}\\
&=- 2\sum_{i,j = 1}^{m}c_ic_j\Big\langle\psi(x_i), \psi(x_j)\Big\rangle_{\mathcal{V}}&\\
&=- 2\left\|\sum_{i= 1}^{m}c_i\psi(x_i)\right\|^2_{\mathcal{V}} \le 0.\vspace{-5cm}
\end{align*}\qedhere
\end{proof}
\vspace{-0.1cm}
\subsection{Test for the Negative Definiteness of a Kernel}
In linear algebra, an $m \times m$ real symmetric matrix $M$ is called \emph{positive semi-definite} if  $\vec{c}^TM\vec{c} \ge 0$ for all $\vec{c} \in \mathbb{R}^m$ and \emph{negative semi-definite} if $\vec{c}^TM\vec{c} \le 0$ for all $\vec{c} \in \mathbb{R}^m$. These conditions are respectively equivalent to $M$ having non-negative eigenvalues and non-positive eigenvalues. Furthermore, a matrix $M$ which has the property $\vec{\hat{c}}^TM\vec{\hat{c}} \le 0$ for all $\vec{\hat{c}} \in \mathbb{R}^m$ with $\sum \hat{c}_i = 0$, where $\hat{c}_i$s are the components of the vector $\vec{\hat{c}}$, is termed \emph{conditionally negative semi-definite}.

Positive (resp. negative) definiteness of a given kernel--a bivariate function--is usually tested by evaluating the positive semi-definiteness (resp. conditionally negative semi-definiteness) of kernel matrices generated with the kernel\footnote{There is an unfortunate confusion of terminology here. A matrix generated by a positive definite kernel is positive semi-definite while a matrix generated by a negative definite kernel is conditionally negative semi-definite.}. Although such a test is not always conclusive if it passes, it is particularly useful to identify non-positive definite and non-negative definite kernels. Given a set of points ${\{x_i\}}_{i = 1}^m \subseteq \mathcal{X}$ and a kernel $k : (\mathcal{X} \times \mathcal{X})\to \mathbb{R}$, the kernel matrix $K$ of the given points has entries $K_{ij} = k(x_i, x_j)$. Any kernel matrix generated by a positive (resp. negative) definite kernel must be positive semi-definite (resp. conditionally negative semi-definite). As noted above, it is straightforward to check the positive semi-definiteness of a kernel matrix $K$ by checking its eigenvalues. However, due to the additional constraint that $\sum \hat{c}_i = 0$, checking the conditionally negative semi-definiteness is not straightforward. We therefore suggest the following procedure.

Let $P = I_m - \frac{1}{m}\vec{1}_m\vec{1}_m^T$, where $I_m$ is the $m \times m$ identity matrix and $\vec{1}_m$ is the $m$-vector of ones. Any given $\vec{\hat{c}} \in \mathbb{R}^m$ with $\sum \hat{c}_i = 0$ can be written as $\vec{\hat{c}} = P\vec{c}$ for some $\vec{c} \in \mathbb{R}^m$. Therefore, the condition $\vec{\hat{c}}^TM\vec{\hat{c}} \le 0$ is equivalent to $\vec{c}^T PMP \vec{c} \le 0$. Hence, we conclude that an $m \times m$ matrix $M$ is conditionally negative semi-definite if and only if $PMP$ is negative semi-definite (i.e., has non-positive eigenvalues). This gives a convenient test for the conditionally negative semi-definiteness of a matrix, which in turn is useful to evaluate the negative definiteness of a given kernel.

\section{Kernels on Manifolds}
\label{sec:kersOnManifolds}
A number of well-known kernels exist for $\Rn$ including the linear kernel, polynomial kernels and the Gaussian RBF kernel. The key challenge in generalizing kernel methods from Euclidean spaces to manifolds lies in defining appropriate positive definite kernels on the manifold. There is no straightforward way to generalize Euclidean kernels such as the linear kernel and polynomial kernels to nonlinear manifolds, since these kernels depend on the linear geometry of $\Rn$. However, we show that the popular Gaussian RBF kernel can be generalized to manifolds under certain conditions.

In this section, we first introduce a general theorem that provides necessary and sufficient conditions to define a positive definite Gaussian RBF kernel on a given manifold and then show that some popular metrics on $\symd$ and $\grsmn$ yield positive definite Gaussian RBFs on the respective manifolds.

\subsection{The Gaussian RBF Kernel on Metric Spaces}
The Gaussian RBF kernel has proven very effective in Euclidean spaces for a variety of kernel-based algorithms. It maps the data points to an infinite dimensional Hilbert space, which, intuitively, yields a very rich representation of the data. In $\mathbb{R}^n$, the Gaussian kernel can be expressed as $k_G(\mathbf{x},\mathbf{y}) := \exp(-\gamma\|\mathbf{x}-\mathbf{y}\|^2)$, which makes use of the Euclidean distance between two data points $\mathbf{x}$ and $\mathbf{y}$. To define a kernel on a manifold, we would like to replace the Euclidean distance by a more accurate distance measure on the manifold. However, not all geodesic distances yield positive definite kernels. For example, in the case of the unit $n$-sphere embedded in $\mathbb{R}^{n+1}$, $\exp(-\gamma \,d_g^2(x, y))$, where $d_g$ is the usual geodesic distance on the manifold, is not positive definite.

We now state our main theorem which provides the necessary and sufficient conditions to obtain a positive definite Gaussian kernel from a given distance function defined on a generic space.

\begin{thm} \label{thm:main}
Let $(M, d)$ be a metric space and define $k : (M \times M) \to \mathbb{R}$ by $k(x,y) := \exp(-\gamma\,d^2(x,y))$. Then, $k$ is a positive definite kernel for all $\gamma > 0$ if and only if there exists an inner product space $\mathcal{V}$ and a function $\psi : M \to \mathcal{V}$ such that $d(x,y) = \left\|\psi(x) - \psi(y)\right\|_{\mathcal{V}}$.
\end{thm}
\begin{proof}
We first note that positive definiteness of $k(., .)$ for all $\gamma$ and negative definiteness of $d^2(., .)$ are equivalent conditions according to Theorem~\ref{thm:sho}.

To prove the forward direction of the present theorem, let us first assume that $\mathcal{V}$ and $\psi$ exist such that $d(x,y) = \left\|\psi(x) - \psi(y)\right\|_{\mathcal{V}}$. Then, from Lemma~\ref{thm:innerprod}, $d^2$ is negative definite and hence $k$ is positive definite for all $\gamma$.

On the other hand, if $k$ is positive definite for all $\gamma$, then $d^2$ is negative definite. Furthermore, $d(x, x) = 0$ for all $x \in M$ since $d$ is a metric. Following Theorem~\ref{thm:negDefEmbedding}, then $\mathcal{V}$ and $\psi$ exist such that $d(x,y) = \left\|\psi(x) - \psi(y)\right\|_{\mathcal{V}}$.\qedhere
\end{proof}

\subsection{Geodesic Distances and the Gaussian RBF}
A Riemannian manifold, when considered with its geodesic distance, forms a metric space. Given Theorem~\ref{thm:main}, it is then natural to wonder under which conditions would a geodesic distance on a manifold yield a Gaussian RBF kernel. We now present and prove the following theorem, which answers this question for complete Riemannian manifolds.

\begin{thm}
\label{thm:geoGaussian}
Let $\mnd$ be a complete Riemannian manifold and $d_g$ be the geodesic distance induced by its Riemannian metric. The Gaussian RBF kernel $k_g: (\mnd \times \mnd) \to \mathbb{R} : k_g(x, y) := \exp(-\gamma\,d_g^2(x, y))$ is positive definite for all $\gamma > 0$ if and only if $\mnd$ is isometric (in the Riemannian sense) to some Euclidean space $\Rn$.
\end{thm}

\begin{proof}
If $\mnd$ is isometric to some $\Rn$, the geodesic distance on $\mnd$ is simply the Euclidean distance in $\Rn$, which can be trivially shown to yield a positive definite Gaussian RBF kernel by setting $\psi$ in Theorem~\ref{thm:main} to the identity function. 

On the other hand, if $k_g$ is positive definite, from Theorem~\ref{thm:main}, there exists an inner product space $\mathcal{V}_g$ and a function $\psi_g : \mnd \to \mathcal{V}_g$ such that $d_g(x, y) = \|\psi_g(x) - \psi_g(y) \|_{\mathcal{V}_g}$. Let $\Hg$ be the completion of $\mathcal{V}_g$. Therefore, $\Hg$ is a Hilbert space, in which $\mathcal{V}_g$ is dense. 

Now, take any two points $x_0, x_1$ in $\mnd$. Since the manifold is complete, from the Hopf-Rinow theorem, there exists a geodesic $\delta(t)$ joining them with $\delta(0) = x_0$ and $\delta(1) = x_1$, and realizing the geodesic distance. By definition, $\delta(t)$ has a constant speed $d_g(x_0, x_1)$. Therefore, for all $x_t = \delta(t)$ where $t \in [0, 1]$, the following equality holds
$$
d_g(x_0, x_t) + d_g(x_t, x_1) = d_g(x_0, x_1).
$$

This must also be true for images $\psi(x_t)$ in $\Hg$ for $t \in [0,1]$. However, since $\Hg$ is a Hilbert space, this is only possible if all the points $\psi(x_t)$ lie on a straight line in $\Hg$. Let $\psi(\mnd)$ be the range of $\psi$. From the previous argument, for any two points in $\psi(\mnd) \subseteq \Hg$, the straight line segment joining them is also in $\psi(\mnd)$. Therefore, $\psi(\mnd)$ is a convex set in $\Hg$. Now, since $\mnd$ is complete, any geodesic must be extensible indefinitely. Therefore, the corresponding line segment in $\psi(\mnd)$ must also be extensible indefinitely. This proves that $\psi(\mnd)$ is an affine subspace of $\Hg$, which is isometric to $\mathbb{R}^n$, for some $n$. Since $\mnd$ is isometric to $\psi(\mnd)$, this proves that $\mnd$ is isometric to the Euclidean space $\mathbb{R}^n$.
\end{proof}

According to Theorem~\ref{thm:geoGaussian}, it is possible to obtain a positive definite Gaussian kernel from the geodesic distance on a Riemannian manifold only when the manifold is made essentially equivalent to some $\Rn$ by the Riemannian metric that defines the geodesic distance. Although this is possible for some Riemannian manifolds, such as the Riemannian manifold of SPD matrices, for some others, it is theoretically impossible.

In particular, if the manifold is compact, it is impossible to find an isometry between the manifold and $\Rn$, since $\Rn$ is not compact. Therefore, it is not possible to obtain a positive definite Gaussian from the geodesic distance of a compact manifold. In such cases, the best hope is to find a different non-geodesic distance on the manifold that does not differ much from the geodesic distance, but still satisfies the conditions of Theorem~\ref{thm:main}. 

\subsection{Kernels on $\symd$}
\label{sec:kerOnSymd}
We now discuss different metrics on $\symd$ that can be used to define positive definite Gaussian kernels. Since $\symd$ is not compact, as explained in the previous section, there is some hope in finding a geodesic distance on it that also defines a positive definite Gaussian kernel. In this section we show that the log-Euclidean distance, which has been proved to be a geodesic distance on $\symd$~\cite{Arsign05TechReport}, is such a distance.


In the log-Euclidean framework, a geodesic connecting $S_1, S_2 \in \symd$ is defined as $\gamma(t) = \exp((1-t)\log(S_1) + t\log(S_2))$ for $t \in [0, 1]$. The log-Euclidean geodesic distance between $S_1$ and  $S_2$ can be expressed as
\begin{equation}
\label{eq:log_dist}
d_{LE}(S_1, S_2) = {\|\log(S_1) - \log(S_2)\|}_F\;,
\end{equation}
where $\|\cdot\|_F$ denotes the Frobenius matrix norm induced by the Frobenius matrix inner product ${\langle.,.\rangle}_F$.

The log-Euclidean distance has proven an effective distance measure on $\symd$~\cite{Arsigny06, Dryden09thestatistical}. Furthermore, it yields a positive definite Gaussian kernel as stated in the following corollary to Theorem~\ref{thm:main}:

\begin{cor}[Theorem~\ref{thm:main}] \label{cor:leGK}
The \textbf{Log-Euclidean Gaussian kernel} $k_{LE} : (\symd \times \symd) \to \mathbb{R} : k_{LE}(S_1,S_2) := \exp(-\gamma\,d_{LE}^2(S_1,S_2))$, where $d_{LE}(S_1, S_2)$ is the log-Euclidean distance between $S_1$ and $S_2$, is a positive definite kernel for all $\gamma > 0$.
\end{cor}
\begin{proof}Directly follows from Theorem~\ref{thm:main} with the Frobenius matrix inner product.
\end{proof}

A number of other metrics have been proposed for $\symd$~\cite{Dryden09thestatistical}. The definitions and properties of these metrics are summarized in Table~\ref{tbl:metricsOnSymd}. Note that only some of them were derived by considering the Riemannian geometry of the manifold and hence define true geodesic distances. Similarly to the log-Euclidean metric, from Theorem~\ref{thm:main}, it directly follows that the Cholesky and power-Euclidean metrics also define positive definite Gaussian kernels for all values of $\gamma$. Note that some metrics may yield a positive definite Gaussian kernel for some value of $\gamma$ only. This, for instance, was shown in~\cite{Sra12} for the root Stein divergence metric. No such result is known for the affine-invariant metric. Constraints on $\gamma$ are nonetheless undesirable, since one should be able to freely tune $\gamma$ to reflect the data distribution, and automatic model selection algorithms require kernels to be positive definite for continuous values of $\gamma > 0$~\cite{Chapelle02}.

{
\renewcommand{\arraystretch}{1.4}
\begin{table*}[t]
\centering
\begin{tabular}{| >{\centering\arraybackslash}m{3.3cm} | >{\centering\arraybackslash}m{6.5cm} | >{\centering\arraybackslash}m{2.3cm} | >{\centering\arraybackslash}m{3.5cm} |}
\hline    
Metric Name & Formula & Geodesic Distance & Positive Definite Gaussian Kernel for all $\gamma > 0$
\tabularnewline
\hline   
\textbf{Log-Euclidean} & ${\|\log(\mathbf{S}_1) - \log(\mathbf{S}_2)\|}_F$ & \textbf{Yes} & \textbf{Yes} 
\tabularnewline
Affine-Invariant & ${\|\log(\mathbf{S}_1^{-1/2} \mathbf{S}_2 \mathbf{S}_1^{-1/2})\|}_F$ & Yes & No
\tabularnewline
Cholesky & ${\| \operatorname{chol}(\mathbf{S}_1) -  \operatorname{chol}(\mathbf{S}_2)\|}_F$ & No & Yes
\tabularnewline
Power-Euclidean & ${ \frac{1}{\alpha} \| \mathbf{S}_1^\alpha - \mathbf{S}_2^\alpha \|}_F$ & No & Yes
\tabularnewline
Root Stein Divergence & $ {\left[ {\log \operatorname{det} \left( \frac{1}{2}\mathbf{S}_1 + \frac{1}{2}\mathbf{S}_2 \right) - \frac{1}{2} \log \operatorname{det} (\mathbf{S}_1 \mathbf{S}_2)} \right]}^{1/2} $ & No & No
\tabularnewline
\hline
\end{tabular}
\vspace{-0.1cm}
\caption{\small {\bf Properties of different metrics on $\symd$.} We analyze the positive definiteness of Gaussian kernels generated by different metrics. Theorem~\ref{thm:main} applies to the metrics claimed to generate positive definite Gaussian kernels. For the other metrics, examples of non-positive definite Gaussian kernels exist.}
\label{tbl:metricsOnSymd}
\vspace{-0.2cm}
\end{table*}
}

{
\renewcommand{\arraystretch}{1.4}
\begin{table*}[t]
\centering
\begin{tabular}{| >{\centering\arraybackslash}m{3.3cm} | >{\centering\arraybackslash}m{6.5cm} | >{\centering\arraybackslash}m{2.3cm} | >{\centering\arraybackslash}m{3.5cm} |}
\hline    
Metric Name & Formula & Geodesic Distance & Positive Definite Gaussian Kernel for all $\gamma > 0$
\tabularnewline
\hline   
\textbf{Projection} & $2^{-1/2}{\|Y_1Y_1^T - Y_2Y_2^T\|}_F = (\sum_i\sin^2\theta_i)^{1/2}$ & No & \textbf{Yes}
\tabularnewline
Arc length & $(\sum_i \theta_i^2)^{1/2}$ & Yes & No
\tabularnewline
Fubini-Study & $\arccos|\det(Y_1^TY_2)| = \arccos(\prod_i{\cos\theta_i)}$ & No & No
\tabularnewline
Chordal 2-norm & ${\|Y_1U - Y_2V\|}_2 = 2\max_i\sin \frac{1}{2}\theta_i$& No & No
\tabularnewline
Chordal F-norm & ${\|Y_1U - Y_2V\|}_F = 2(\sum_i\sin^2 \frac{1}{2}\theta_i)^{1/2}$ & No & No
\tabularnewline
\hline
\end{tabular}
\vspace{-0.1cm}
\caption{\small {\bf Properties of different metrics on $\grsmn$.} Here, $USV^T$ is the singular value decomposition of $Y_1^TY_2$, whereas $\theta_i$s are the the principal angles between the two subspaces $[Y_1]$ and $[Y_2]$. Our claims on positive definiteness are supported by either Theorem~\ref{thm:main}, or counter-examples.}
\label{tbl:metricsOnGrsmn}
\vspace{-0.5cm}
\end{table*}
}

\subsection{Kernels on $\grsmn$}
\label{sec:kerOnGrsmn}
Similarly to $\symd$, different metrics can be defined on $\grsmn$. Many of these metrics are related to the principal angles between two subspaces. Given two $n \times r$ matrices $Y_1$ and $Y_2$ with orthonormal columns, representing two points on $\grsmn$, the principal angles between the corresponding subspaces are obtained from the singular value decomposition of $Y_1^TY_2$~\cite{Edelman98}. More specifically, if $USV^T$ is the singular value decomposition of $Y_1^TY_2$, then the entries of the diagonal matrix $S$ are the cosines of the principal angles between $[Y_1]$ and $[Y_2]$. Let $\{\theta_i\}_{i=1}^r$ represent those principal angles. Then, the geodesic distance derived from the canonical geometry of the Grassmann manifold, called the arc length, is given by $(\sum_i \theta_i^2)^{1/2}$~\cite{Edelman98}. Unfortunately, as can be shown with a counter-example, this distance, when squared, is not negative definite and hence does not yield a positive definite Gaussian for all $\gamma > 0$. Given Theorem~\ref{thm:geoGaussian} and the discussion that followed, this not a surprising result: Since the Grassmann manifold is a compact manifold, it is not possible to find a geodesic distance that also yields a positive definite Gaussian for all $\gamma > 0$. 

Nevertheless, there exists another widely used metric on the Grassmann manifold, namely the projection metric, which gives rise to a positive definite Gaussian kernel. The projection distance between two subspaces $[Y_1], [Y_2]$ is given by

\begin{equation}
\label{eq:projDist}
d_{P}([Y_1], [Y_2]) = 2^{-1/2}{\|Y_1Y_1^T - Y_2Y_2^T\|}_F.
\end{equation}

We now formally introduce this Gaussian RBF kernel on the Grassmann manifold.

\begin{cor}[Theorem~\ref{thm:main}] \label{cor:projGK}
The \textbf{Projection Gaussian kernel} $k_{P} : (\grsmn \times \grsmn) \to \mathbb{R} : k_{P}([Y_1],[Y_2]) := \exp(-\gamma\,d_{P}^2([Y_1],[Y_2]))$, where $d_{P}([Y_1], [Y_2])$ is the projection distance between $[Y_1]$ and $[Y_2]$, is a positive definite kernel for all $\gamma > 0$.
\end{cor}
\begin{proof} Follows from Theorem~\ref{thm:main} with the Frobenius matrix inner product.
\end{proof}

As shown in Table~\ref{tbl:metricsOnGrsmn}, none of the other popular metrics on Grassmann manifolds have this property. Counter-examples exist for these metrics to support our claims.

\subsubsection{Calculation of the Projection Gaussian Kernel}
To calculate the kernel introduced in Corollary~\ref{cor:projGK}, one needs to calculate the squared projection distance given by $d_{P}^2([Y_1], [Y_2]) = 2^{-1}{\|Y_1Y_1^T - Y_2Y_2^T\|}_F^2$ or $d_{P}^2([Y_1], [Y_2]) = (\sum_i\sin^2\theta_i)^{1/2}$ where both $Y_1$ and $Y_2$ are $n \times r$ matrices. The second formula requires a singular value decomposition to find the $\theta_i$s, which we would like to avoid due to its computational complexity. The first formula requires calculating the Frobenius $l_2$ distance between $Y_1Y_1^T$ and $Y_2Y_2^T$, both of which are $n \times n$ matrices. For many applications, $n$ is quite large, and therefore, the direct computation of $\|Y_1Y_1^T - Y_2Y_2^T\|_F^2$ is inefficient.
As a consequence, the cost of computing $d_{P}^2$ can be reduced with the following equation, which makes use of some properties of the matrix trace and of the fact that $Y_1^TY_1 = Y_2^TY_2 = I_r$:
\begin{align*}
&d_{P}^2([Y_1], [Y_2]) = 2^{-1}{\|Y_1Y_1^T - Y_2Y_2^T\|}_F^2= r - \|Y_1^TY_2\|_F^2 \;.
\end{align*}
This implies that it is sufficient to compute only the Frobenius norm of $Y_1^TY_2$, an $r \times r$ matrix, to calculate $d_{P}^2([Y_1], [Y_2])$.

\section{Kernel-based Algorithms on Manifolds}
\label{sec:algos}
A major advantage of being able to compute positive definite kernels on manifolds is that it directly allows us to make use of algorithms developed for $\mathbb{R}^n$, while still accounting for the geometry of the manifold. In this section, we discuss the use of five kernel-based algorithms on manifolds. The resulting algorithms can be thought of as generalizations of the original Euclidean kernel methods to manifolds. In the following, we use $\mnd$, $k(., .)$, ${\mathcal{H}}$ and ${{\phi}}(x)$ to denote a manifold, a positive definite kernel defined on $\mnd \times \mnd$, the RKHS generated by $k$, and the feature vector in $\mathcal{H}$ to which $x \in \mnd$ is mapped, respectively. Although we use ${{\phi}}(x)$ for explanation purposes, following the kernel trick, it never needs to be explicitly computed in any of the algorithms.

\subsection{Classification on Manifolds}

We first consider the binary classification problem on a manifold. To this end, we propose to extend the popular Euclidean kernel SVM algorithm to manifold-valued data. Given a set of training examples ${\{(x_i, y_i)\}}_{i=1}^m$, where $x_i \in \mnd$ and the label $y_i \in \{-1, 1\}$, kernel SVM searches for a hyperplane in $\mathcal{H}$ that separates the feature vectors of the positive and negative classes with maximum margin. The class of a test point $x \in \mnd$ is determined by the position of the feature vector ${{\phi}}(x)$ in $\mathcal{H}$ relative to the separating hyperplane. Classification with kernel SVM can be done very fast, since it only requires to evaluate the kernel at the \emph{support vectors}.

The above procedure is equivalent to solving the standard kernel SVM problem with kernel matrix generated by $k$. Thus, any existing SVM software package can be utilized for training and classification. Convergence of standard SVM optimization algorithms is guaranteed, since $k$ is positive definite.
 
Kernel SVM on manifolds is much simpler to implement and less computationally demanding in both training and testing phases than the current state-of-the-art binary classification algorithms on manifolds, such as LogitBoost on a manifold~\cite{Tuzel08}, which involves iteratively combining weak learners on different tangent spaces. Weighted mean calculation in LogitBoost on a manifold involves an extremely expensive gradient descent procedure at each boosting iteration, which makes the algorithm scale poorly with the number of training samples. Furthermore, while LogitBoost learns classifiers on tangent spaces used as first order Euclidean approximations to the manifold, our approach works in a rich high dimensional feature space. As will be shown in our experiments, this yields better classification results. 

With manifold-valued data, extending the current state-of-the-art binary classification methods to multi-class classification is not straight-forward~\cite{Tosato10, CaseiroCVPR2013_Rolling}. By contrast, our manifold kernel SVM classification method can easily be extended to the multi-class case with standard one-vs-one or one-vs-all procedures.

\subsection{Feature Selection on Manifolds}
\label{sec:mkl}

We next tackle the problem of combining multiple manifold-valued descriptors via a Multiple Kernel Learning (MKL) approach. The core idea of MKL is to combine kernels computed from different descriptors (e.g., image features) to obtain a kernel that optimally separates two classes for a given classifier. Here, we follow the formulation of~\cite{Varma07} and make use of an SVM classifier. As a feature selection method, MKL has proven more effective than conventional techniques such as wrappers, filters and boosting~\cite{Varma09}.

More specifically, given training examples ${\{(x_i, y_i)\}}_1^m$, where $x_i \in \mathcal{X}$ (some nonempty set), $y_i \in \{-1,1\}$, and a set of descriptor generating functions ${\{g_j\}}_1^N$ where $g_j : \mathcal{X} \to \mnd$, we seek to learn a binary classifier $f : \mathcal{X} \to \{-1, 1\}$ by selecting and optimally combining the different descriptors generated by $g_1,\ldots, g_N$. Let ${K}^{(j)}$ be the kernel matrix generated by $g_j$  and $k$ as ${K}_{pq}^{(j)} = k(g_j(x_p),g_j(x_q))$. The combined kernel can be expressed as ${K}^{*} = \sum_j{\lambda_j}{K}^{(j)}$, where $\lambda_j \ge 0$ for all $j$ guarantees the positive definiteness of ${K}^{*}$. The weights {\boldmath$\lambda$} can be learned using a min-max optimization procedure with an $l_1$ regularizer on {\boldmath$\lambda$} to obtain a sparse combination of kernels. The algorithm has two steps in each iteration: First it solves a conventional SVM problem with {\boldmath$\lambda$}, hence ${K}^{*}$, fixed. Then it updates {\boldmath$\lambda$} with the SVM parameters fixed. These two steps are repeated until convergence. For more details, we refer the reader to \cite{Varma07} and \cite{Varma09}. Note that convergence of MKL is only guaranteed if all the kernels are positive definite, which is satisfied in this setup since $k$ is positive definite.

We also note that MKL on manifolds gives a convenient method to combine manifold-valued descriptors with Euclidean descriptors which is otherwise a difficult task due to their different geometries. 

\subsection{Dimensionality Reduction on Manifolds}

We next study the extension of kernel PCA to nonlinear dimensionality reduction on manifolds. The usual Euclidean kernel PCA has proven successful in many applications~\cite{Scholkopf98, Scholkopf2002book}. On a manifold, kernel PCA proceeds as follows: All points $x_i \in \mnd$ of a given dataset ${\{x_i\}}_{i=1}^{m}$ are mapped to feature vectors in ${\mathcal{H}}$, thus yielding the transformed set ${\{\phi(x_i)\}}_{i=1}^{m}$. The covariance matrix of this transformed set is then computed, which really amounts to computing the kernel matrix of the original data using the function $k$. An $l$-dimensional representation of the data is obtained by computing the eigenvectors of the kernel matrix. This representation can be thought of as a Euclidean representation of the original manifold-valued data. However, owing to our kernel, it was obtained by accounting for the geometry of $\mnd$.

Once the kernel matrix is calculated with $k$, implementation details of the algorithm are similar to that of the Euclidean kernel PCA algorithm. We refer the reader to~\cite{Scholkopf98} for more details on the implementation.



\subsection{Clustering on Manifolds}
\label{sec:kernelKmeans}

For clustering problems on manifolds, we propose to make use of kernel $k$-means. Kernel $k$-means maps points to a high-dimensional Hilbert space and performs $k$-means clustering in the resulting feature space~\cite{Scholkopf98}. In a manifold setting, a given dataset ${\{x_i\}}_{i=1}^{m}$, with each $x_i \in \mnd$, is clustered into a pre-defined number of groups in ${\mathcal{H}}$, such that the sum of the squared distances from each $\phi(x_i)$ to the nearest cluster center is minimized. The resulting clusters can then act as classes for the ${\{x_i\}}_{i=1}^{m}$. 

The unsupervised clustering method on Riemannian manifolds proposed in~\cite{Goh08} clusters points in a low-dimensional space after dimensionality reduction on the manifold. In contrast, our method performs clustering in a high-dimensional RKHS which, intuitively, better represents the data distribution.

\subsection{Discriminant Analysis on Manifolds}

The kernelized version of linear discriminant analysis, known as Kernel Fisher Discriminant Analysis (Kernel FDA), can also be extended to manifolds on which a positive definite kernel can be defined. Given ${\{(x_i, y_i)\}}_{i=1}^{m}$, with each $x_i \in \mnd$ having class label $y_i$, manifold kernel FDA maps each point $x_i$ on the manifold to a feature vector in $\mathcal{H}$ and finds a new basis in $\mathcal{H}$ where the class separation is maximized. The output of the algorithm is a Euclidean representation of the original manifold-valued data, but with a larger separation between class means and a smaller within-class variance. Up to $(l - 1)$ dimensions can be extracted via kernel FDA where $l$ is the number of classes. We refer the reader to~\cite{Mika1999, Scholkopf2002book} for implementation details. In Euclidean spaces, kernel FDA has become an effective pre-processing step to perform nearest-neighbor classification in the highly discriminative, reduced dimensional space.

\begin{figure}[t!]  
\vspace{-0.2cm}
  \centering
  \includegraphics[trim = 2mm 57mm 2mm 65mm, clip, width=0.5\textwidth]{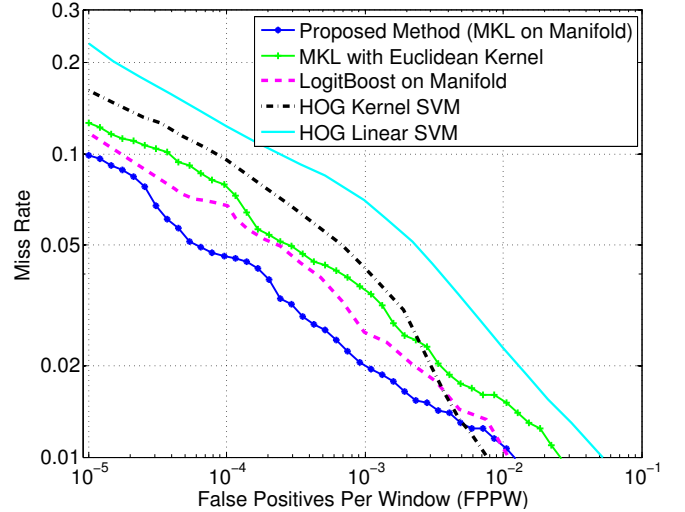}

  \caption{\small {\bf Pedestrian detection.} Detection-Error tradeoff curves for the proposed manifold MKL approach and state-of-the-art methods on the INRIA dataset. Our method outperforms existing manifold methods and Euclidean kernel methods.
  The curves for the baselines were reproduced from~\cite{Tuzel08}.}
  \label{fig:detCurves}
\vspace{-0.4cm}
\end{figure}

\section{Applications and Experiments}
We now present two series of experiments on $\symd$ and $\grsmn$ using the positive definite kernels introduced in Section~\ref{sec:kersOnManifolds} and the algorithms described in Section~\ref{sec:algos}.

\subsection{Experiments on $\symd$}
In the following, we use the log-Euclidean Gaussian kernel defined in Corollary~\ref{cor:leGK} to apply different kernel methods to $\symd$. We compare our kernel methods on $\symd$ to other state-of-the-art algorithms on Riemannian manifolds and to kernel methods on $\symd$ with the usual Euclidean Gaussian kernel that does not account for the nonlinearity of the manifold.
\subsubsection{Pedestrian Detection}

We first demonstrate the use of our log-Euclidean Gaussian kernel for the task of pedestrian detection with kernel SVM and MKL on $\symd$. 
Let ${\{(W_i, y_i)\}}_{i=1}^m$ be the training set, where each $W_i \in \mathbb{R}^{h \times w}$ is an image window and $y_i \in \{-1, 1\}$ is the class label (background or person) of $W_i$. Following~\cite{Tuzel08}, we use covariance descriptors computed from the feature vector $\left[x,\;y,\;\lvert I_x \rvert,\; \lvert I_y \rvert,\; \sqrt{I_x^2 + I_y^2},\; \lvert I_{xx} \rvert,\; \lvert I_{yy} \rvert,\;\arctan{\left(\frac{\lvert I_x \rvert }{\lvert I_y \rvert}\right)}\right]$, where $x, y$ are pixel locations and $I_x, I_y, \ldots$ are intensity derivatives. The covariance matrix for an image patch of arbitrary size therefore is an $8 \times 8$ SPD matrix. In an $h \times w$ window $W$, a large number of covariance descriptors can be computed from subwindows with different sizes and positions sampled from $W$. We consider $N$ subwindows ${\{w_j\}}_{j=1}^N$ of size ranging from $h/5 \times w/5$ to $h \times w$, positioned at all possible locations. The covariance descriptor of each subwindow is normalized using the covariance descriptor of the full window to improve robustness against illumination changes. Such covariance descriptors can be computed efficiently using integral images~\cite{Tuzel08}.

\begin{table*}[t]
    \begin{subtable}{0.3\textwidth}
        \centering
        {\renewcommand{\arraystretch}{0.7}
				\begin{tabular}{ c@{\hspace{3pt}} c@{\hspace{3pt}}  c@{\hspace{3pt}} c}
			\includegraphics[width=0.19\textwidth]{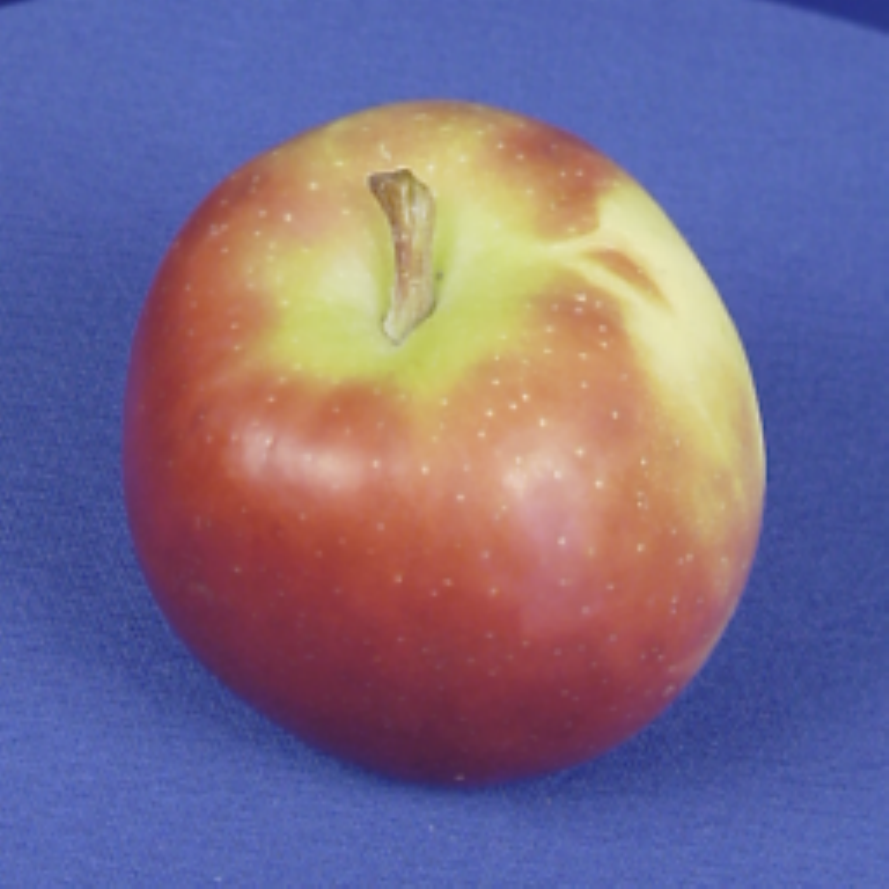}%
			& \includegraphics[width=0.19\textwidth]{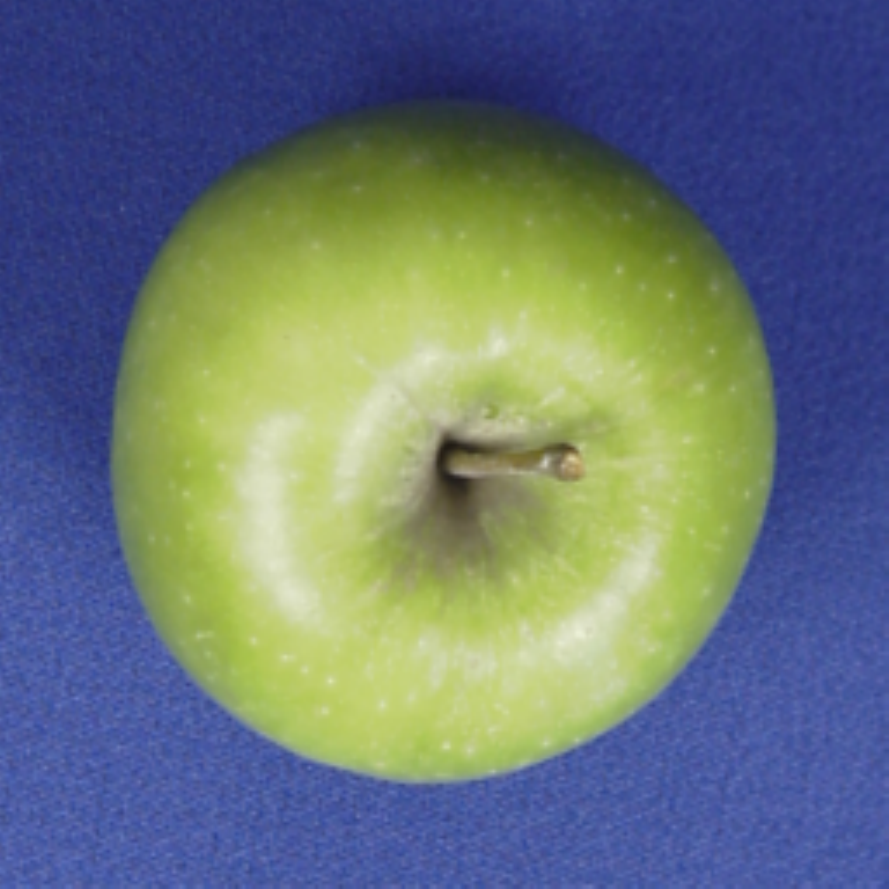}%
			& \includegraphics[width=0.19\textwidth]{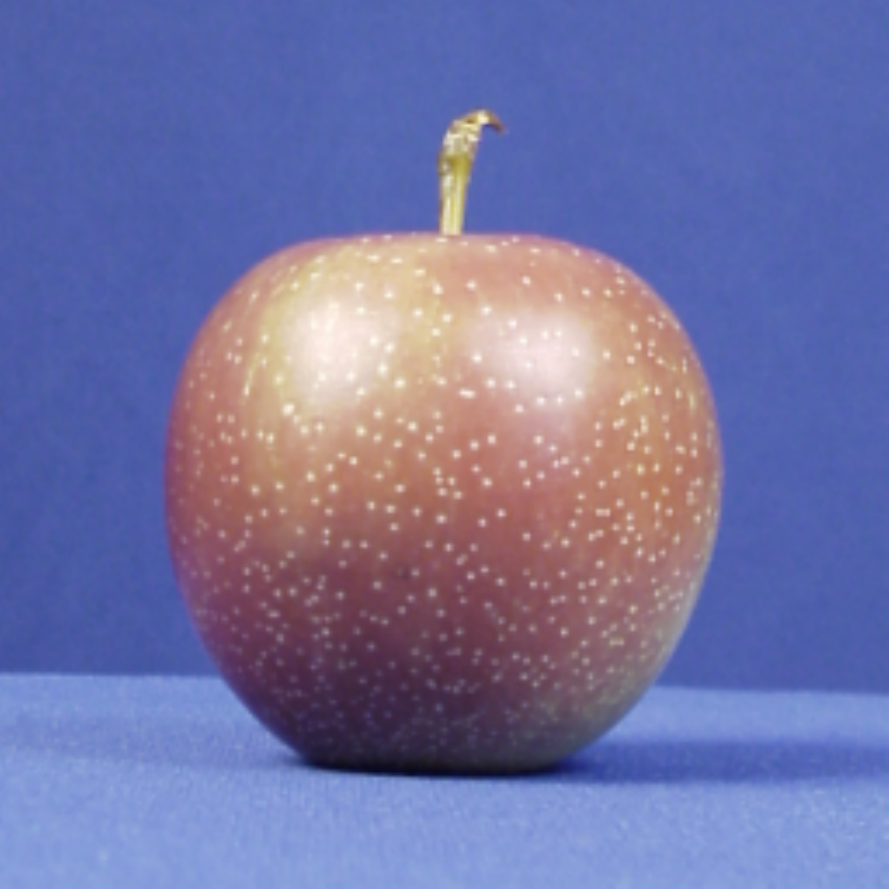}%
			& \includegraphics[width=0.19\textwidth]{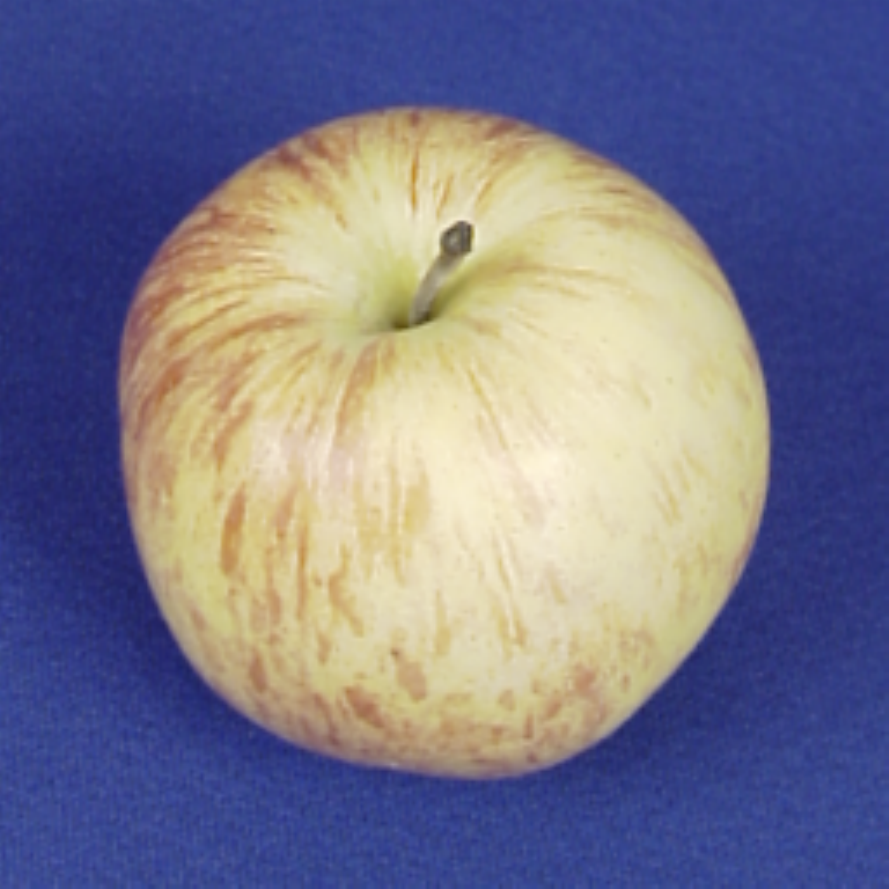}\\
					\newline
			\includegraphics[width=0.19\textwidth]{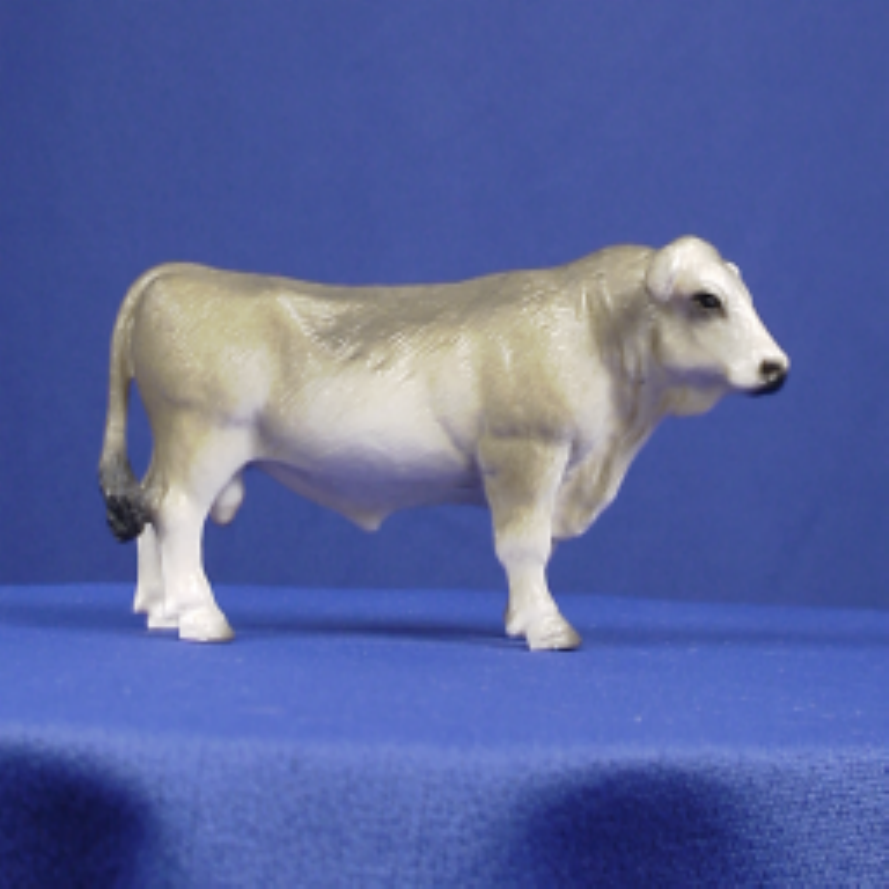}%
			& \includegraphics[width=0.19\textwidth]{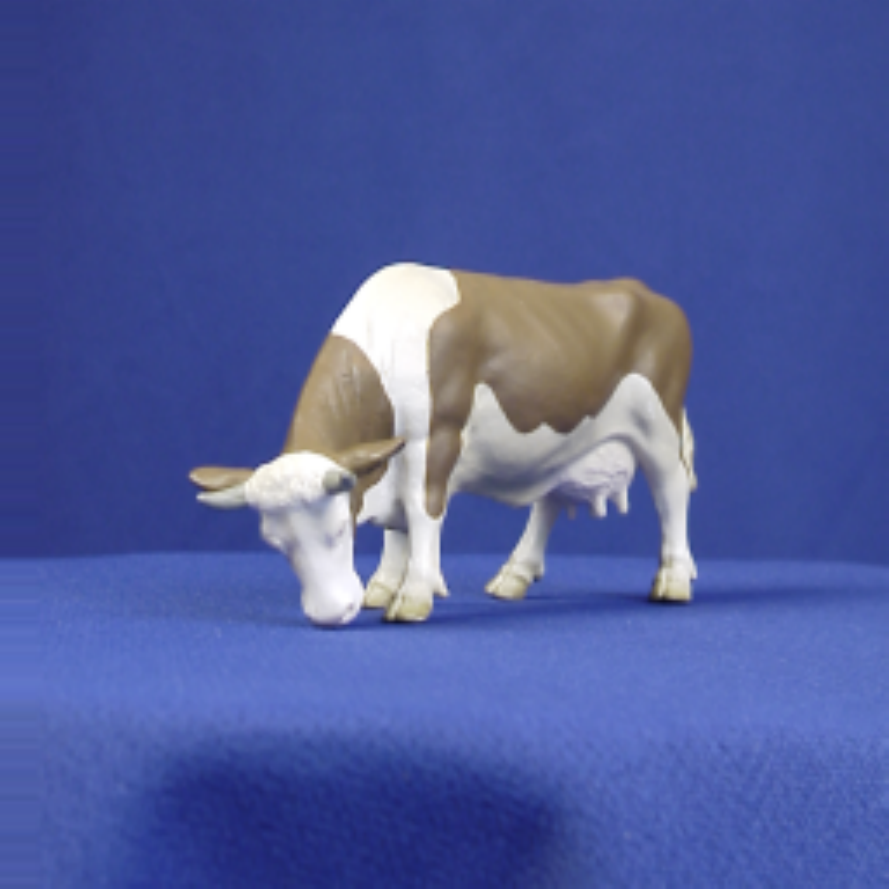}%
			& \includegraphics[width=0.19\textwidth]{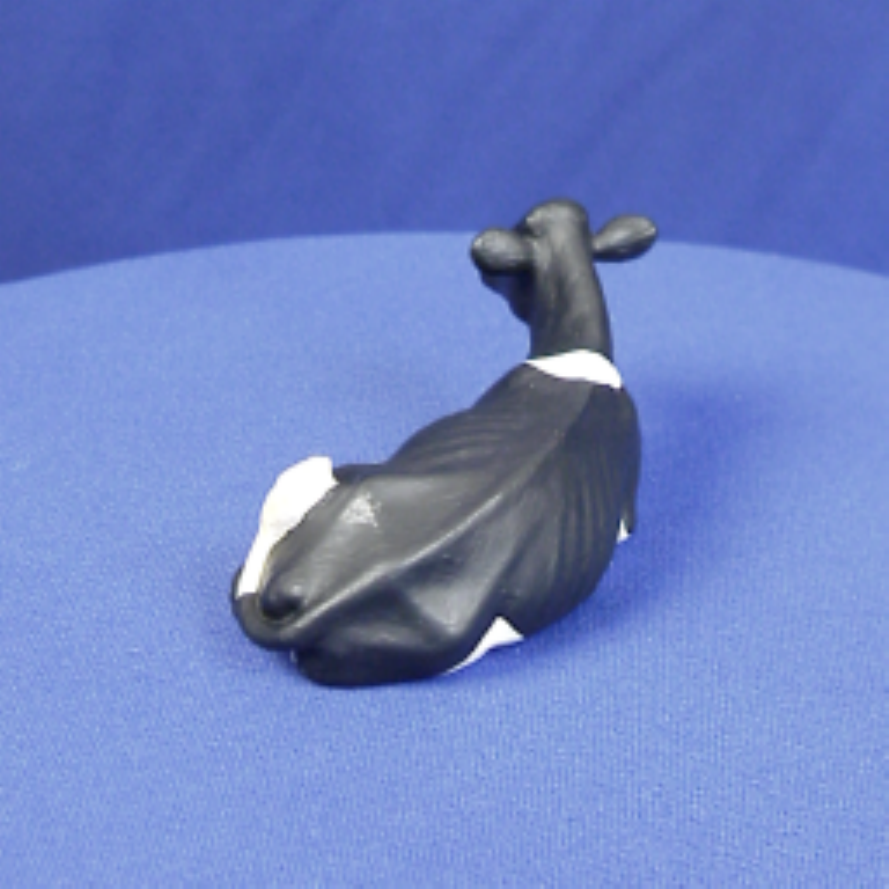}%
			& \includegraphics[width=0.19\textwidth]{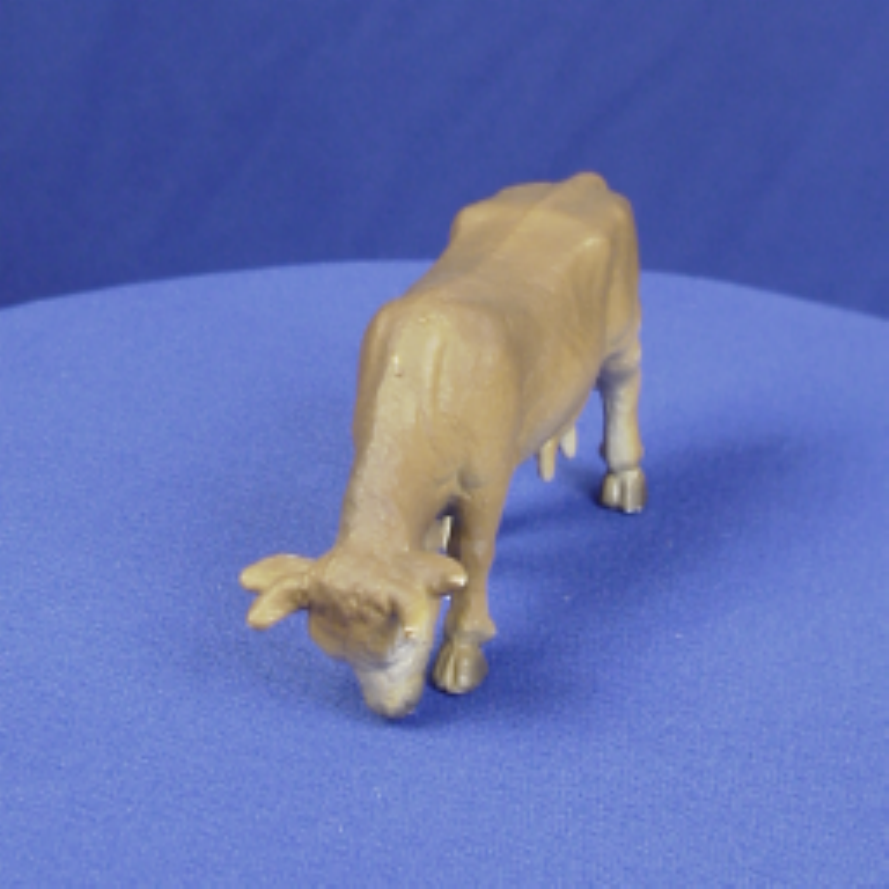}\\
				\newline
			\includegraphics[width=0.19\textwidth]{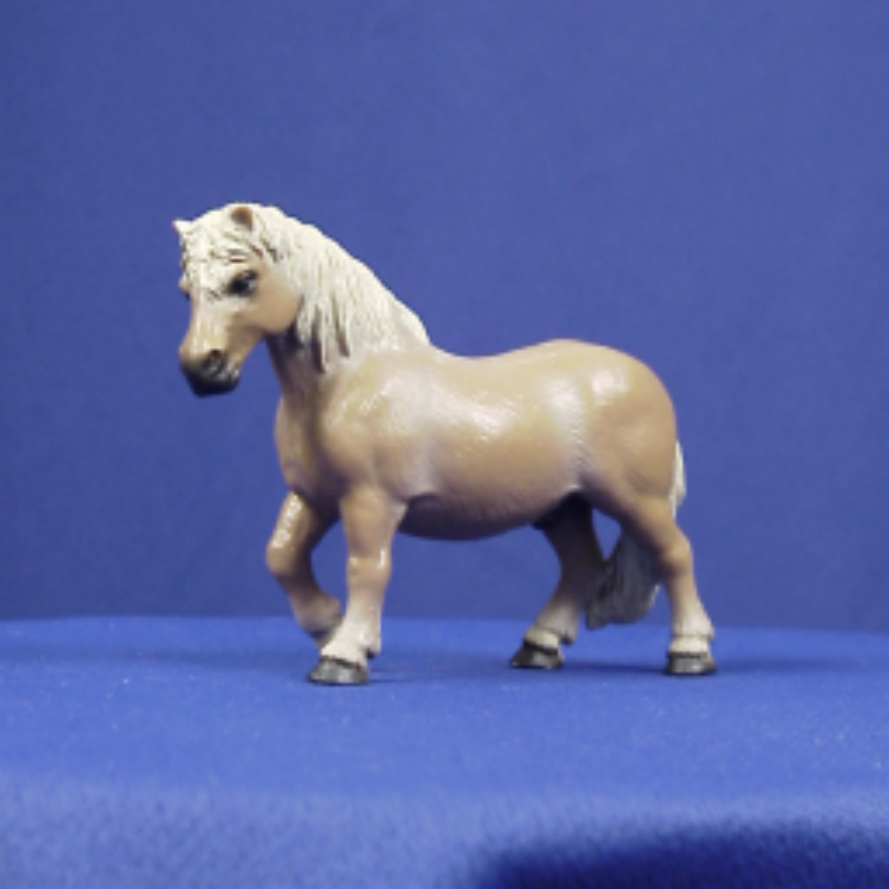}%
			& \includegraphics[width=0.19\textwidth]{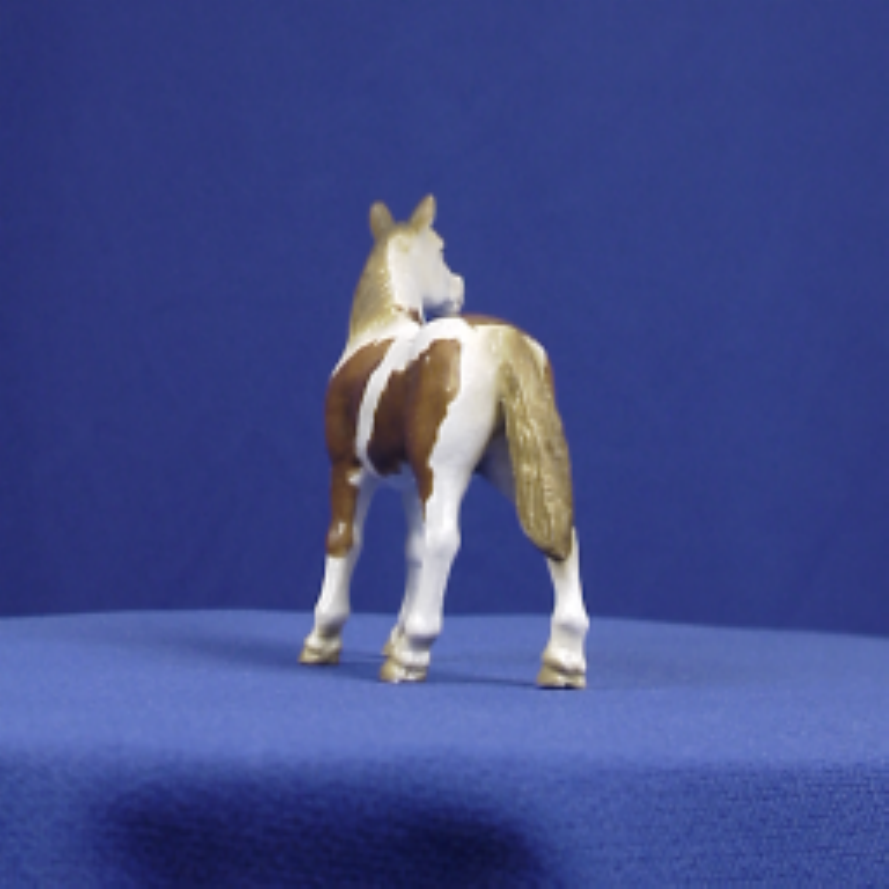}%
			& \includegraphics[width=0.19\textwidth]{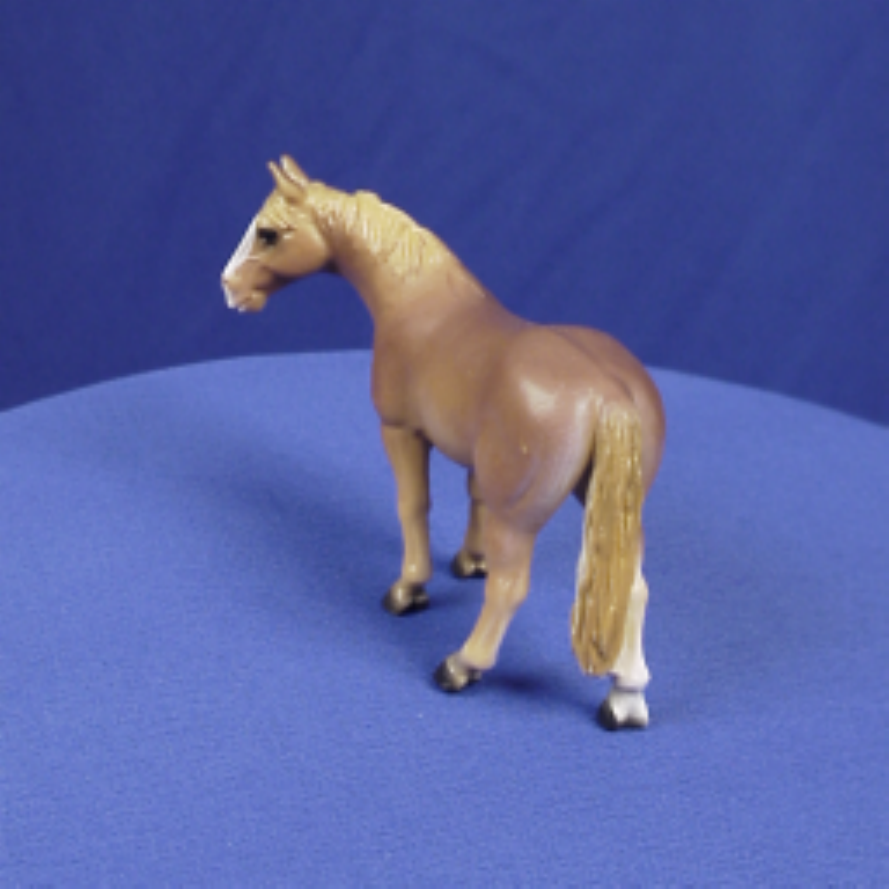}%
			& \includegraphics[width=0.19\textwidth]{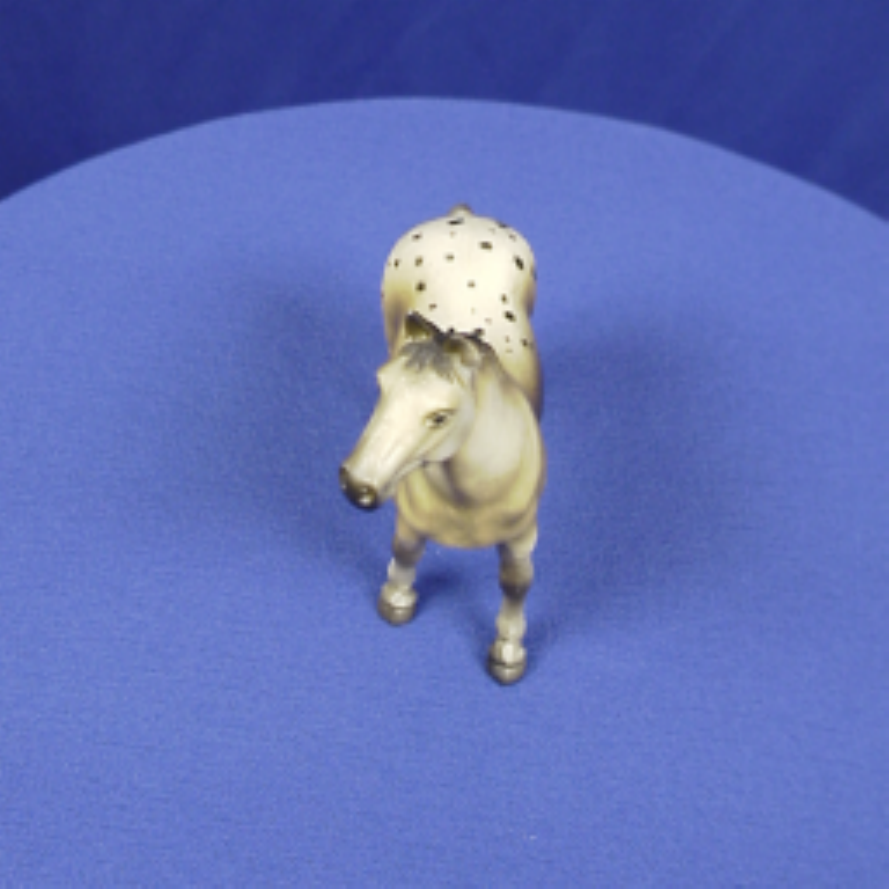}\\
				\end{tabular}				
				}
    \end{subtable}
    \begin{subtable}{0.46\textwidth}
        \centering
        {\renewcommand{\arraystretch}{1.2}
				\begin{tabular}{| >{\centering\arraybackslash} m{30pt} | p{20pt} | p{20pt} | p{20pt} | p{20pt} | p{20pt} | p{20pt} | p{20pt} | p{20pt}|}
					\hline
						\multirow{2}{*}{\vspace{-0.1cm}Nb. of\vspace{0.2cm}} & \multicolumn{2}{ @{}>{\centering}p{72pt}@{}|}{Euclidean}& \multicolumn{2}{@{}>{\centering}p{72pt}@{}|}{Cholesky} & \multicolumn{2}{@{}>{\centering}p{72pt}@{}|}{Power-Euclidean} & \multicolumn{2}{@{}>{\centering}p{72pt}@{}|}{Log-Euclidean} \\
						\cline{2-9}
						  classes & KM & KKM & KM & KKM & KM & KKM & KM & KKM\\
						\hline
						3     & 72.50 & 79.00      & 73.17 & 82.67     & 71.33 & 84.33 & 75.00 & {\bf 94.83}\\
						
						4     & 64.88 & 73.75      & 69.50 & 84.62     & 69.50 & 83.50 & 73.00 & {\bf 87.50}\\
						
						5     & 54.80 & 70.30      & 70.80 & 82.40     & 70.20 & 82.40 & 74.60 & {\bf 85.90} \\						
						6     & 50.42 & 69.00      & 59.83 & 73.58     & 59.42 & 73.17 & 66.50 & {\bf 74.50}\\						
						7     & 42.57 & 68.86      & 50.36 & 69.79     & 50.14 & 69.71 & 59.64 & {\bf 73.14}\\						
						8     & 40.19 & 68.00      & 53.81 & 69.44     & 54.62 & 68.44 & 58.31 & {\bf 71.44}\\
						\hline 
				\end{tabular}
				}
    \end{subtable}
    \vspace{-0.1cm}
    \caption{\small {\bf Object categorization.} Sample images and percentages of correct clustering on the ETH-80 dataset using $k$-means (KM) and kernel $k$-means (KKM) with different metrics on $\symd$. The proposed KKM method with the log-Euclidean metric achieves the best results in all the tests.}
    \label{tbl:ethResults}
    \vspace{-0.2cm}
\end{table*}

Let $X^{(j)}_i \in \operatorname{Sym}^+_8$ denote the covariance descriptor of the $j^{\text{th}}$ subwindow of $W_i$. To reduce this large number of descriptors, we pick the best 100 subwindows that do not mutually overlap by more than $75\%$, by ranking them according to their variance across all training samples. The rationale behind this ranking is that a good descriptor should have a low variance across all given positive detection windows. 
Since the descriptors lie on a manifold, for each $X^{(j)}$, we compute a variance-like statistic $\operatorname{var}(X^{(j)})$ across all positive training samples as
\begin{equation}
\label{eq:var}
\operatorname{var}(X^{(j)}) = \frac{1}{m_+} \sum_{i:y_i=1}{ d_g^p( X^{(j)}_i, {\bar{X}}^{(j)})}, 
\end{equation}
where $m_+$ is the number of positive training samples and $\bar{ X }$ is the Karcher mean of ${ \{ X_i \} }_{i:y_i=1}$ given by
$
\bar{X} = \exp\left({\frac{1}{m_+}\sum_{i:y_i=1}{\log(X_i)}}\right)
$
under the log-Euclidean metric. We set $p = 1$ in Eq.\eqref{eq:var} to make the statistic less sensitive to outliers. We then use the SVM-MKL framework described in Section~\ref{sec:mkl} to learn the final classifier, where each kernel is defined on one of the 100 selected subwindows. At test time, detection is achieved in a sliding window manner followed by a non-maximum suppression step.

To evaluate our approach, we made use of the INRIA person dataset~\cite{Dalal05Hog}. Its training set consists of 2,416 positive windows and 1,280 person-free negative images, and its test set of 1,237 positive windows and 453 negative images. Negative windows are generated by sampling negative images~\cite{Dalal05Hog}. We first used all positive samples and 12,800 negative samples (10 random windows from each negative image) to train an initial classifier. We used this classifier to find hard negative examples in the training images, and re-trained the classifier by adding these hard examples to the training set. Cross validation was used to determine the hyperparameters including the parameter $\gamma$ of the kernels. We used the evaluation methodology of~\cite{Dalal05Hog}. 

In Fig.~\ref{fig:detCurves}, we compare the detection-error tradeoff (DET) curves of our approach and state-of-the-art methods. The curve for our method was generated by continuously varying the decision threshold of the final MKL classifier. We also evaluated our MKL framework with the Euclidean Gaussian kernel. Note that the proposed MKL method with our Riemannian kernel outperforms MKL with the Euclidean kernel, as well as LogitBoost on the manifold. This demonstrates the importance of accounting for the nonlinearity of the manifold using an appropriate positive definite kernel. It is also worth noting that LogitBoost on the manifold is significantly more complex and harder to implement than our method.

\subsubsection{Visual Object Categorization}

We next tackle the problem of unsupervised object categorization. To this end, we used the ETH-80 dataset~\cite{Leibe03} which contains 8 categories with 10 objects each and 41 images per object. We used 21 randomly chosen images from each object to compute the parameter $\gamma$ and the rest to evaluate clustering accuracy. For each image, we used a single $5 \times 5$ covariance descriptor calculated from the features $\left[ x,\; y,\; I\;, \lvert I_x \rvert\;, \lvert I_y \rvert\right]$, where $x$, $y$ are pixel locations and $I$, $I_x$, $I_y$ are intensity and derivatives. To obtain object categories, the kernel $k$-means algorithm on $\operatorname{Sym}^+_5$ described in Section~\ref{sec:kernelKmeans} was employed.

One drawback of $k$-means and its kernel counterpart is their sensitivity to initialization. To overcome this, we ran each algorithm $20$ times with random initializations and picked the iteration that converged to the minimum sum of point-to-centroid squared distances. For kernel $k$-means on $\operatorname{Sym}^+_5$, distances in the RKHS were used. We assumed the number of clusters to be known.

To set a benchmark, we evaluated the performance of both $k$-means and kernel $k$-means on $\operatorname{Sym}^+_5$ with different metrics that generate positive definite Gaussian kernels (see Table~\ref{tbl:metricsOnSymd}). For the power-Euclidean metric, we used $\alpha = 0.5$, which, as in the non-kernel method of~\cite{Dryden09thestatistical}, we found to yield the best results. For all non-Euclidean metrics with (non-kernel) $k$-means, the Karcher mean~\cite{Dryden09thestatistical} was used to compute the centroid. The results of the different methods are summarized in Table~\ref{tbl:ethResults}. Manifold kernel $k$-means with the log-Euclidean Gaussian kernel performs significantly better than all other methods in all test cases. These results also outperform the results with the heat kernel reported in~\cite{Caseiro12}. Note, however, that~\cite{Caseiro12} only considered 3 and 4 classes without mentioning which classes were used. 

\subsubsection{Texture Recognition}
We then utilized our log-Euclidean Gaussian kernel to demonstrate the effectiveness of manifold kernel PCA on texture recognition. To this end, we used the Brodatz dataset~\cite{Randen99filtering}, which consists of $111$ different $640 \times 640$ texture images. Each image was divided into four subimages of equal size, two of which were used for training and the other two for testing.

For each training image, covariance descriptors of 50 randomly chosen $128 \times 128$ windows were computed from the feature vector $\left[I,\; \lvert I_x \rvert,\; \lvert I_y \rvert,\; \lvert I_{xx} \rvert\;, \lvert I_{yy} \rvert  \right]$~\cite{Tuzel06}. Kernel PCA on $\operatorname{Sym}^+_ 5$ with our Riemannian kernel was then used to extract the top $l$ principal directions in the RKHS, and project the training data along those directions. Given a test image, we computed 100 covariance descriptors from random windows and projected them to the $l$ principal directions obtained during training. Each such projection was classified using a majority vote over its 5 nearest-neighbors. The class of the test image was then decided by majority voting among the 100 descriptors. Cross validation on the training set was used to determine $\gamma$. For comparison purposes, we repeated the same procedure with the Euclidean Gaussian kernel. Results obtained for these kernels and different values of $l$ are presented in Table~\ref{tbl:textureResults}. The better recognition accuracy indicates that kernel PCA with the Riemannian kernel more effectively captures the information of the manifold-valued descriptors than the Euclidean kernel.

\begin{table}[t]
\centering
{\centering\arraybackslash
        \renewcommand{\arraystretch}{1.2}
\begin{tabular}{c|cccc}
\hline
\multirow{2}{*}{Kernel} & \multicolumn{4}{c}{Classification Accuracy}  \tabularnewline
& $l = 10$ & $l = 11$ & $l = 12$ & $l = 15$  \tabularnewline
\hline
Log-Euclidean & {\bf 95.50}  & {\bf 95.95} & {\bf 96.40} & \bf {96.40}  \tabularnewline
Euclidean     & 89.64 & 90.09 & 90.99 & 91.89  \tabularnewline
\hline
\end{tabular}
}

\caption{\small {\bf Texture recognition.} Recognition accuracies on the Brodatz dataset with $k$-NN in an $l$-dimensional Euclidean space obtained by kernel PCA. The log-Euclidean Gaussian kernel introduced in this paper captures information more effectively than the usual Euclidean Gaussian kernel.}
\label{tbl:textureResults}
\vspace{-0.6cm}
\end{table}

\subsubsection{Segmentation}

We now illustrate the use of our kernel to segment different types of images. First, we consider DTI segmentation, which is a key application area of algorithms on $\symd$. We utilized kernel $k$-means on $\operatorname{Sym}^+_3$ with our Riemannian kernel to segment a real DTI image of the human brain. Each pixel of the input DTI image is a $3 \times 3$ SPD matrix, which can thus directly be used as input to the algorithm. The $k$ clusters obtained by the algorithm act as classes, thus yielding a segmentation of the image.

Fig.~\ref{fig:dti} depicts the resulting segmentation along with the ellipsoid and fractional anisotropy representations of the original DTI image. We also show the results obtained by replacing the Riemannian kernel with the Euclidean one. Note that, up to some noise due to the lack of spatial smoothing, Riemannian kernel $k$-means was able to correctly segment the corpus callosum from the rest of the image. 
%

We then followed the same approach to perform 2D motion segmentation. To this end, we used a spatio-temporal structure tensor directly computed on image intensities (i.e., without extracting features such as optical flow). The spatio-temporal structure tensor for each pixel is computed as $T = K * (\nabla I \nabla I^T)$, where $\nabla I = (I_x, I_y, I_t)$ and $K *$ indicates convolution with the regular Gaussian kernel for smoothing purposes. Each pixel is thus represented as a $3 \times 3$ SPD matrix $T$ and segmentation can be performed by clustering these matrices using kernel $k$-means on $\operatorname{Sym}^+_3$.

We applied this strategy to two images taken from the Hamburg Taxi sequence. Fig.~\ref{fig:oflow} compares the results of kernel $k$-means with our log-Euclidean Gaussian kernel with the results of~\cite{Goh08} obtained by first performing LLE, LE, or HLLE on $\operatorname{Sym}^+_3$ and then clustering in the low dimensional space. Note that our approach yields a much cleaner segmentation than the baselines. This could be attributed to the fact that we perform clustering in a high dimensional feature space, whereas the baselines work in a reduced dimensional space.

\begin{figure}[t]
\vspace{-0.3cm}
\centering
\begin{tabular}{cc}
       		\includegraphics[width=0.145\textwidth, trim = 0mm 0mm 0mm 0mm, clip]{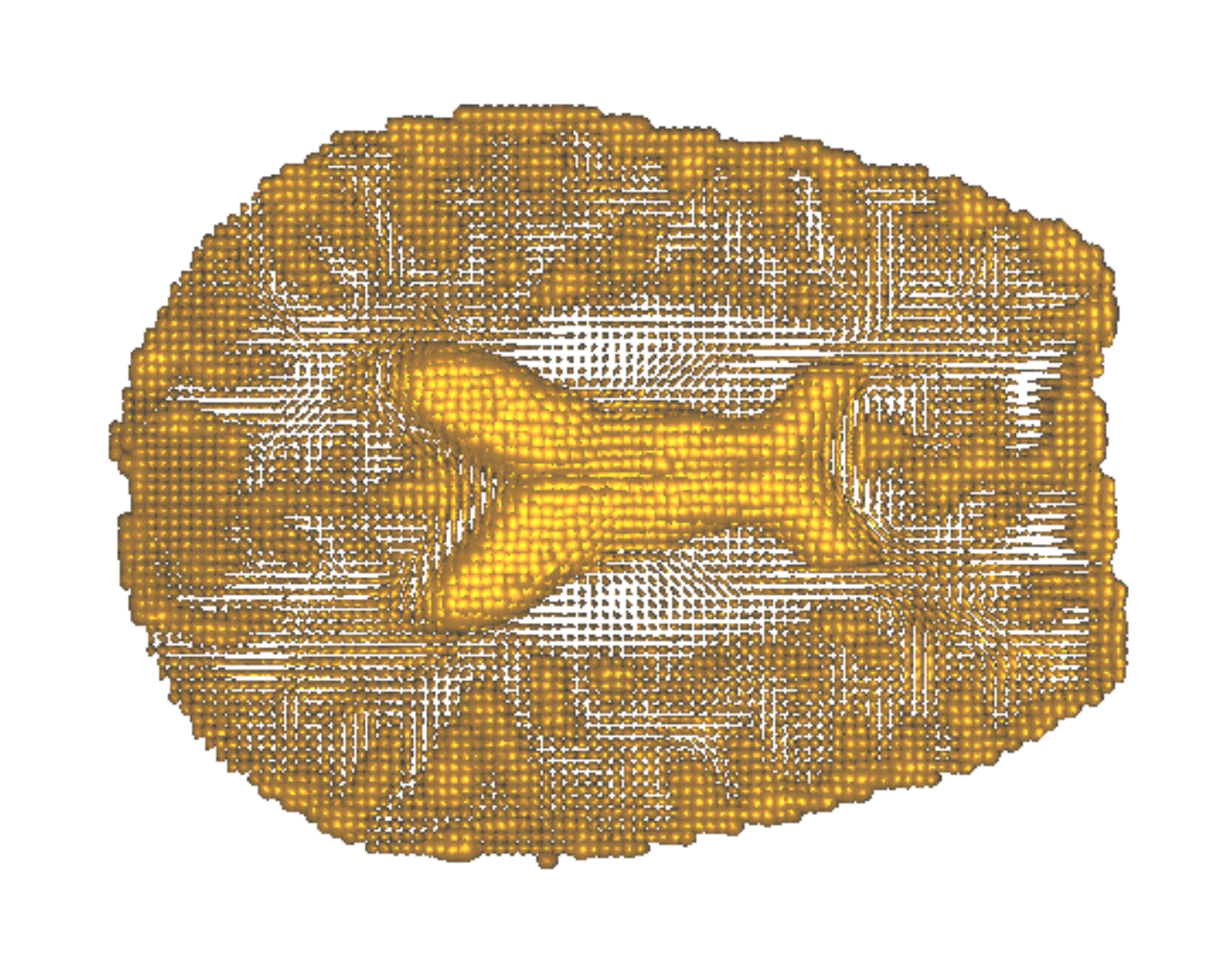} &
		\includegraphics[width=0.145\textwidth, trim = 0mm 0mm 0mm 0mm]{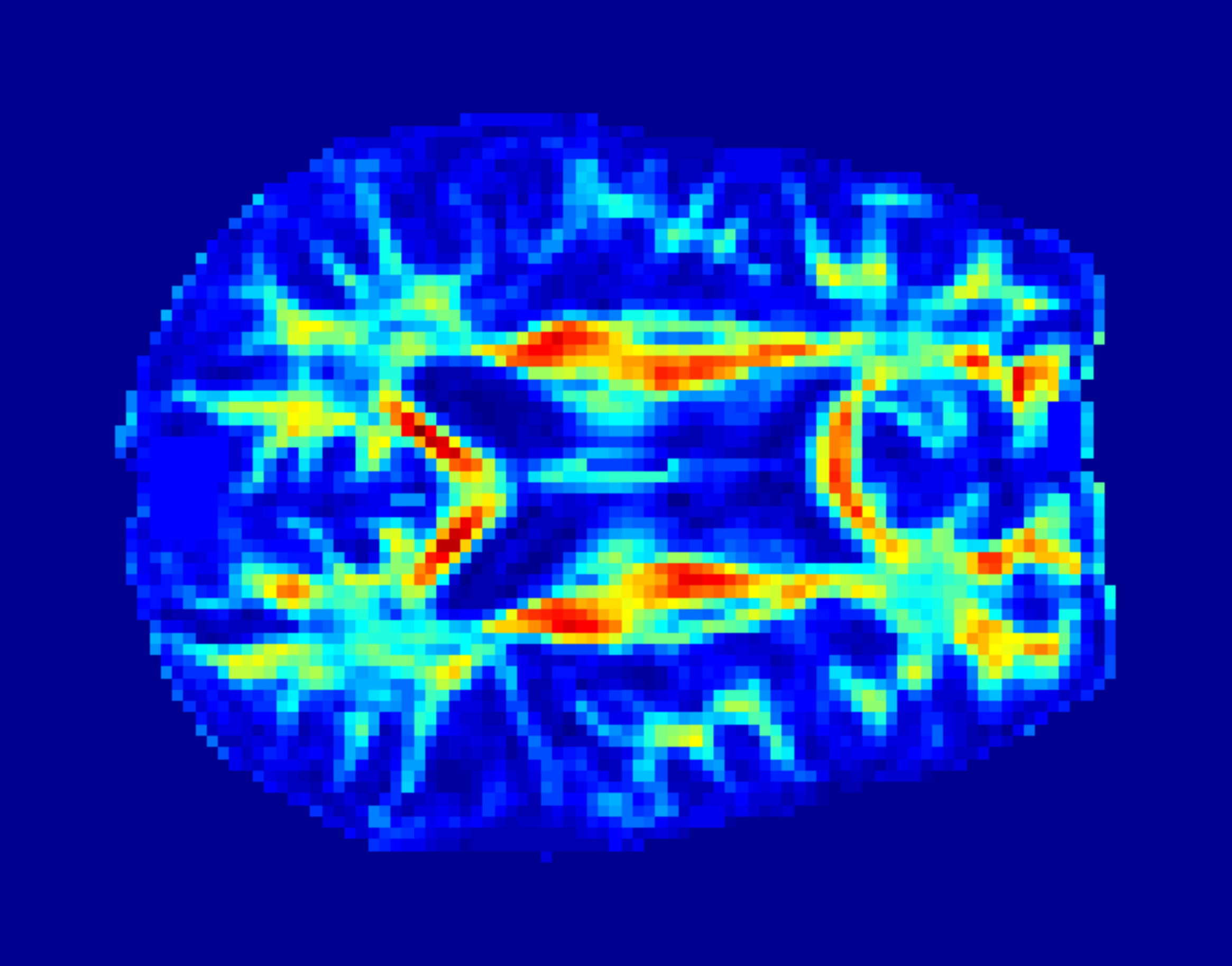} \\
                   Ellipsoids & Fractional Anisotropy \vspace{-0.00cm}\\
                	\includegraphics[width=0.145\textwidth, trim = 0mm 0mm 0mm 0mm]{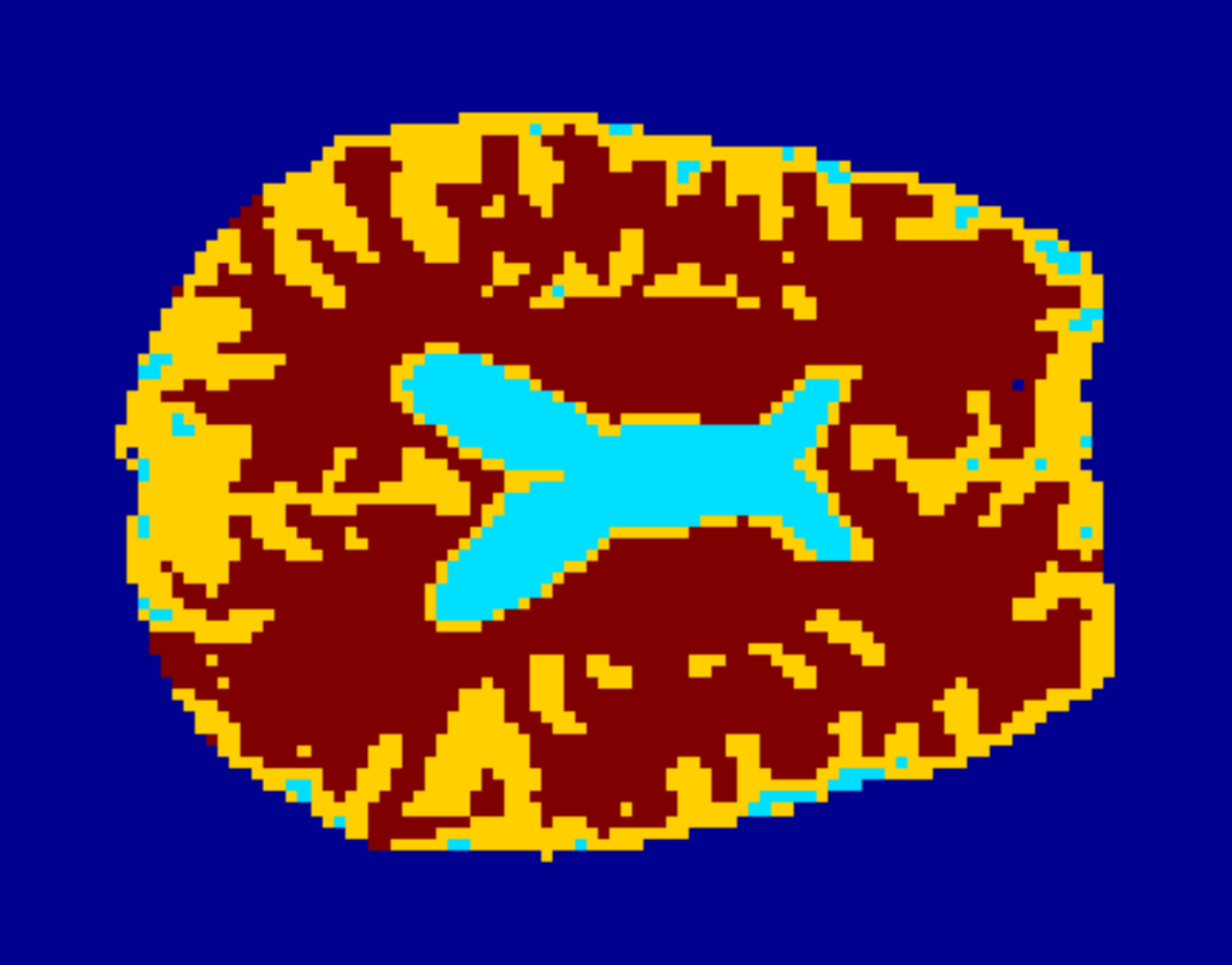} &
		\includegraphics[width=0.145\textwidth, trim = 0mm 0mm 0mm 0mm]{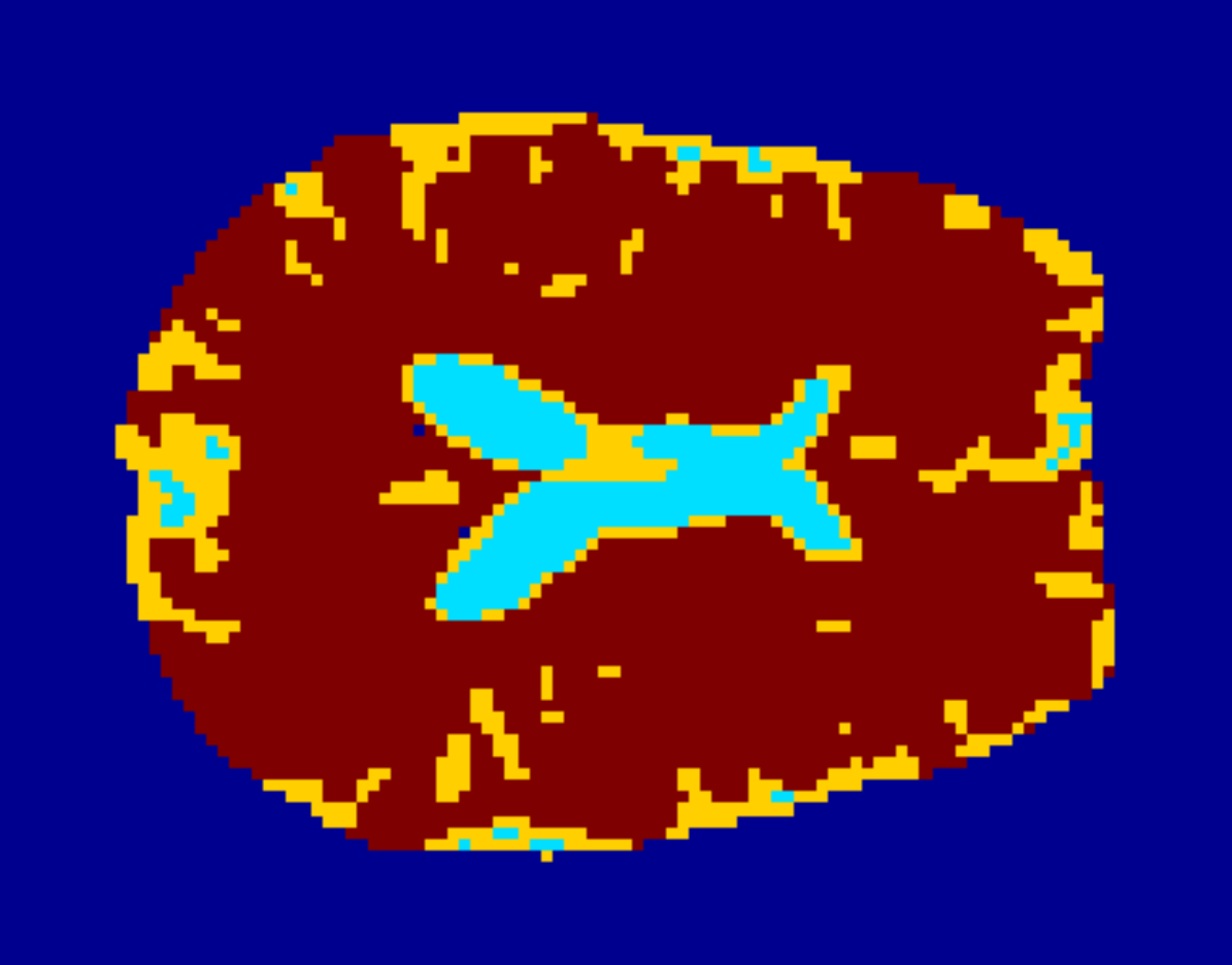} \\		
       Riemannian kernel & Euclidean kernel 
\end{tabular}
        \caption{\small {\bf DTI segmentation.} Segmentation of the corpus callosum with kernel $k$-means on $\operatorname{Sym}^+_3$. The proposed kernel yields a cleaner segmentation.}
        \label{fig:dti}
        \vspace{-0.5cm}
\end{figure}

\begin{figure}[t!]
\centering
\begin{tabular}{ccc}
               \includegraphics[width=0.11\textwidth]{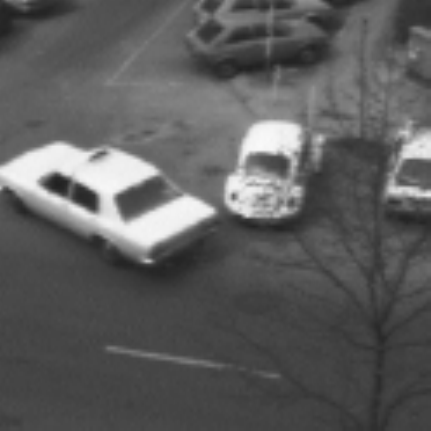} &
               \includegraphics[width=0.11\textwidth]{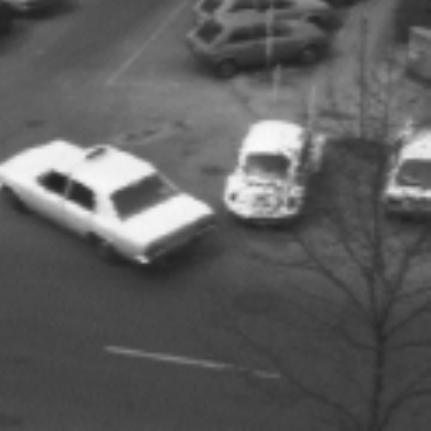} &
	        \frame{\includegraphics[width=0.11\textwidth]{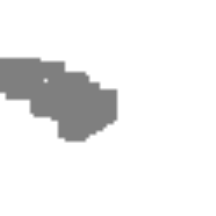}} \\
	        Frame 1 & Frame 2 & KKM on $\operatorname{Sym}^+_3$ \vspace{0.2cm}\\
                \frame{\includegraphics[width=0.11\textwidth]{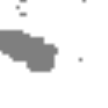}}&
                \frame{\includegraphics[width=0.11\textwidth]{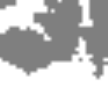}}&
                \frame{\includegraphics[width=0.11\textwidth]{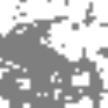}}\\
                LLE on $\operatorname{Sym}^+_3$ & LE on $\operatorname{Sym}^+_3$ & HLLE on $\operatorname{Sym}^+_3$ 
                \end{tabular}
                \vspace{-0.1cm}
        \caption{\small {\bf 2D motion segmentation.} Comparison of the segmentations obtained with kernel $k$-means with our Riemannian kernel (KKM), LLE, LE and HLLE on $\operatorname{Sym}^+_3$. Our KKM algorithm yields a much cleaner segmentation. The baseline results were reproduced from~\cite{Goh08}.}
        \label{fig:oflow}
        \vspace{-0.5cm}
\end{figure}

\vspace{-0.2cm}
\subsection{Experiments on $\grsmn$}
We now present our experimental evaluation of the proposed kernel methods on $\grsmn$ with the projection Gaussian kernel introduced in Corollary~\ref{cor:projGK}. We compare our results to those obtained with state-of-the-art classification and clustering methods on $\grsmn$, and show that we achieve significantly better classification and clustering accuracies in a number of different tasks.

In each of the following experiments, we model an image set with the linear subspace spanned by its principal components. More specifically, let ${\{\vec{f}_i\}}_{i = 1}^p$, with each $\vec{f}_i \in \Rn$, be a set of descriptors each representing one image in a set of $p$ images. The image set can then be represented by the linear subspace spanned by $r$ $(r < p, n)$ principal components of ${\{\vec{f}_i\}}_{i = 1}^p$. Limiting $r$ helps reducing noise and other fine variations within the image set, which are not useful for classification or clustering. The image set descriptors obtained in this manner are $r$-dimensional linear subspaces of the $n$-dimensional Euclidean space, which lie on the $(n, r)$ Grassmann manifold $\grsmn$. We note that the $r$-principal components of ${\{\vec{f}_i\}}_{i = 1}^p$ can be efficiently obtained by performing a Singular Value Decomposition (SVD) on the $n \times p$ matrix $F$ having the $\vec{f}_i$s as columns: If $USV^T$ is the singular value decomposition of $F$, the columns of $U$ corresponding to the largest $r$ singular values give the $r$ principal components of ${\{\vec{f}_i\}}_{i = 1}^p$.

\vspace{-0.2cm}
\subsubsection{Video Based Face Recognition}

Face recognition from video, which uses an image set for identification of a person, is a rapidly developing area in computer vision. For this task, the videos are often modeled as linear subspaces, which lie on a Grassmann manifold~\cite{Hamm08, Harandi2011}. To demonstrate the use of our projection Gaussian kernel in video based face recognition, we used the YouTube Celebrity dataset~\cite{Kim2008}, which contains 1910 video clips of 47 different people. Face recognition on this dataset is challenging since the videos are highly compressed and most of them have low resolution.

We used the cascaded face detector of \cite{Viola04} to extract face regions from videos and resized them to have a common size of $96 \times 96$. Each face image was then represented by a histogram of Local Binary Patterns~\cite{Ojala2002} having $232$ equally-spaced bins. We next represented each image set corresponding to a single video clip by a linear subspace of order 5. We randomly chose $70\%$ of the dataset for training and the remaining $30\%$ for testing. We report the classification accuracy averaged over 10 different random splits. 

We employed both kernel SVM on a manifold and kernel FDA on a manifold with our projection Gaussian kernel. With kernel SVM, a one-vs-all procedure was used for multiclass classification. For kernel FDA, the training data was projected to an $(l - 1)$ dimensional space, where $l = 47$ is the number of classes, and we used a $1$-nearest-neighbor method to predict the class of a test sample projected to the same space. We determined the hyperparameters of both methods using cross-validation on the training data.

We compared our approach with several state-of-the-art image set classification methods: Discriminant analysis of Canonical Correlations (DCC)~\cite{Kim2007}, Kernel Affine Hull Method (KAHM)~\cite{Cevikalp2010}, Grassmann Discriminant Analysis (GDA)~\cite{Hamm08}, and Graph-embedding Grassmann Discriminant Analysis (GGDA)~\cite{Harandi2011}. As shown in Table~\ref{tbl:YTC_Results}, manifold kernel SVM achieves the best accuracy. GDA uses the Grassmann projection kernel which corresponds to the linear kernel with FDA in a Euclidean space. Therefore, the results of GDA and Manifold Kernel FDA in Table~\ref{tbl:YTC_Results} also provide a nice comparison between the linear kernel and our projection Gaussian kernel. To obtain a similar comparison with SVMs, we also performed one-vs-all SVM classification with the projection kernel, which really amounts to linear SVM in the Projective space. This method is denoted by Linear SVM in Table~\ref{tbl:YTC_Results}. With both FDA and SVM, our Gaussian kernel performs better than the projection kernel, thus agreeing with the observation in Euclidean spaces that the Gaussian kernel performs better than the linear kernel.

\begin{table}[t]
\vspace{-0.3cm}
\renewcommand{\arraystretch}{1.3}
\centering
{\centering\arraybackslash
\begin{tabular}{c|> \centering p{64pt}| > \centering p{69pt}}
\hline
\multirow{2}{*}{Method} & Face Recognition Accuracy & Action Recognition Accuracy \tabularnewline
\hline
DCC~\cite{Kim2007} 			& 60.21 $\pm$ 2.9 & 41.95 $\pm$ 9.6\tabularnewline
KAHM~\cite{Cevikalp2010}	& 67.49 $\pm$ 3.5 & 70.05 $\pm$ 0.9\tabularnewline
GDA~\cite{Hamm08}			& 58.72 $\pm$ 3.0&	67.33 $\pm$ 1.1\tabularnewline
GGDA~\cite{Harandi2011}		& 61.05 $\pm$ 2.2& 73.54 $\pm$ 2.0\tabularnewline
Linear SVM					& 64.76 $\pm$ 2.1 & 74.66 $\pm$ 1.2\tabularnewline
\textbf{Manifold Kernel FDA}& 65.32 $\pm$ 1.4& 76.35 $\pm$ 1.0\tabularnewline
\textbf{Manifold Kernel SVM}&\textbf{71.78} $\pm$ \textbf{2.4} & \textbf{76.95} $\pm$ \textbf{0.9}\tabularnewline
\hline
\end{tabular}
}
\caption{\small {\bf Face and action recognition.} Recognition accuracies on the YouTube Celebrity and Ballet datasets. Our manifold kernel SVM method achieves the best results.}
\label{tbl:YTC_Results}
\vspace{-0.3cm}
\end{table}

\vspace{-0.2cm}
\subsubsection{Action Recognition}

We next demonstrate the benefits of our projection Gaussian kernel on action recognition. To this end, we used the Ballet dataset~\cite{Wang2009}, which contains 44 videos of 8 actions performed by 3 different actors. Each video contains different actions performed by the same actor. Action recognition on this dataset is challenging due to large intra-class variations in clothing and movement.

We grouped every 6 subsequent frames containing the same action, which resulted in 2338 image sets. Each frame was described by a Histogram of Oriented Gradients (HOG) descriptor~\cite{Dalal05Hog}, and 4 principal components were used to represent each image set. The samples were randomly split into two equally-sized sets to obtain training and test data. We report the average accuracy over 10 different splits. 

As in the previous experiment, we used kernel FDA and one-vs-all SVM with the projection Gaussian kernel and compared our methods with DCC, KAHM, GDA, GGDA, and Linear SVM. As evidenced by the results in Table~\ref{tbl:YTC_Results}, Manifold Kernel FDA and Manifold Kernel SVM both achieve superior performance compared to the state-of-the-art algorithms. Out of the two methods proposed in this paper, Manifold Kernel SVM achieves the highest accuracy.

\vspace{-.3cm}
\subsubsection{Pose Categorization}
Finally, we demonstrate the use of our projection Gaussian kernel in clustering on the Grassmann manifold. To this end, we used the well-known CMU-PIE face dataset~\cite{Sim2002}, which contains face images of 67 subjects with 13 different poses and 21 different illuminations. Images were closely cropped to enclose the face region followed by a resizing to $32 \times 32$. The vectorized intensity values were directly used to describe each image. We used images of the same subject with the same pose but different illuminations to form an image set, which was then represented by a linear subspace of order 6. This resulted in a total of $67 \times 13 = 871$ Grassmann descriptors, each lying on $\mathcal{G}^6_{1024}$. 

The goal of the experiment was to cluster together image sets having the same pose. We randomly divided the images of the same subject with the same pose into two equally sized sets to obtain two collections of 871 image sets. The optimum value for $\gamma$ was determined with the first collection and the results are reported on the second one. We compare our Manifold Kernel $k$-means~(MKKM) with two other algorithms on the same data. The first algorithm, proposed in \cite{Turaga11}, is the conventional $k$-means with the arc length distance on $\grsmn$ and the corresponding Karcher mean~(KM-AL). The publicly available code of \cite{Turaga11} was used to obtain the results with KM-AL. Since the Karcher mean with the arc length distance does not have a closed-form solution and has to be calculated using a gradient descent procedure, KM-AL tends to slow down with the dimensionality of the Grassmann descriptors and the number of samples. The second baseline was $k$-means with the projection metric and the corresponding Karcher mean (KM-PM). Although the Karcher mean can be calculated in closed-form in this case, the algorithm becomes slow when working with large matrices. In our setup, since the projection space consisted of symmetric matrices of size $1024 \times 1024$, projection $k$-means boils down to performing $k$-means in a $1024 \times (1024 + 1) / 2 = 524,800$ dimensional space.

The results of the three clustering algorithms for different numbers of clusters are given in Table~\ref{tbl:PIE_Results}. 
As in the experiment on $\symd$, each clustering algorithm was run $20$ times with different random initializations, and the iteration which converged to the minimum energy was picked. Note that our MKKM algorithm yields the best performance in all scenarios. Performance gain with the MKKM algorithm becomes more significant when the problem gets harder with more clusters.

\begin{table}[t]
\vspace{-0.3cm}
\renewcommand{\arraystretch}{1.3}
\centering
{\centering\arraybackslash
\begin{tabular}{c|c|c|c}
\hline
\multirow{2}{*}{Nb. of Clusters} & \multicolumn{3}{c}{Clustering Accuracy}\tabularnewline	
& KM-AL~\cite{Turaga11} & KM-PM & MKKM\tabularnewline
\hline
5 & 88.06 & 94.62 & \textbf{96.12}		\tabularnewline
6 & 85.07 & 94.52 & \textbf{95.27}		\tabularnewline
7 & 85.50 & 94.88 & \textbf{95.52}		\tabularnewline
8 & 85.63 & 90.93 & \textbf{95.34}		\tabularnewline
9 & 73.96 & 79.10 & \textbf{83.08}		\tabularnewline
10& 70.30 &78.95 & \textbf{81.79}		\tabularnewline
11& 68.38 & 78.56& \textbf{81.41}		\tabularnewline
12& 64.55 & 74.75& \textbf{81.22}		\tabularnewline
13& 61.65 & 73.82& \textbf{80.14}		\tabularnewline
\hline
\end{tabular}
}
\caption{\small {\bf Pose grouping.} Clustering accuracies on the CMU-PIE dataset. The proposed kernel $k$-means method yields the best results in all the test cases.\vspace{-0.4cm}}
\label{tbl:PIE_Results}
\end{table}

\vspace{-0.4cm}
\section{Conclusion}
\vspace{-0.1cm}
In this paper, we have introduced a unified framework to analyze the positive definiteness of the Gaussian RBF kernel defined on a manifold or a more general metric space. We have then used the same framework to derive provably positive definite kernels on the Riemannian manifold of SPD matrices and on the Grassmann manifold. These kernels were then utilized to extend popular learning algorithms designed for Euclidean spaces, such as SVM and FDA, to manifolds. Our experimental evaluation on several challenging computer vision tasks, such as pedestrian detection, object categorization, segmentation, action recognition and face recognition, has evidenced that identifying positive definite kernel functions on manifolds can be greatly beneficial when working with manifold-valued data. In the future, we intend to study this problem for other nonlinear manifolds, as well as to extend our theory to more general kernels than the Gaussian RBF kernel.

\vspace{-0.2cm}
\section*{Acknowledgements}
\vspace{-0.1cm}
This work was supported in part by an ARC grant. The authors would like thank Bob Williamson for useful discussions.


%


\ifCLASSOPTIONcaptionsoff
  \newpage
\fi



\vspace{-0.4cm}
\bibliographystyle{IEEEtran}
\bibliography{IEEEabrv,kom_bib}
%

%

\vspace{-1.2cm}
\begin{IEEEbiography}[{\includegraphics[width=1in,height=1.25in, clip,keepaspectratio]{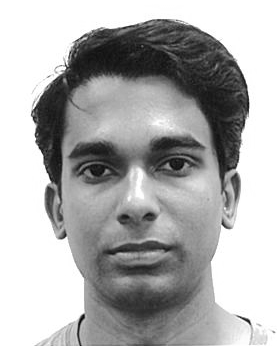}}]
{Sadeep Jayasumana}
obtained his BSc degree in Electronics and Telecommunication Engineering from University of Moratuwa, Sri Lanka, in 2010. He
is now a PhD student at the College of Engineering and Computer Science, Australian National University (ANU). He is also a member of the Computer Vision Research Group at NICTA, government funded research laboratory.
He received CSiRA Best Recognition prize at DICTA 2013. He is a student member of IEEE.
\end{IEEEbiography}
\vspace{-1.4cm}
\begin{IEEEbiography}[{\includegraphics[width=1in,height=1.25in, clip,keepaspectratio]{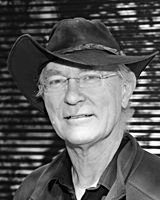}}]{Richard Hartley} is a member of the computer vision group in the Research School of Engineering, at ANU, where he has been since January, 2001. He is also a member of the computer vision research group in NICTA.
He worked at the GE Research and Development Center from 1985 to 2001, working first in VLSI design, and later in computer vision.  He became involved with Image Understanding and Scene Reconstruction working with GE's Simulation and Control Systems Division.  
He is an author (with A. Zisserman) of the book Multiple View Geometry in Computer Vision.
\end{IEEEbiography}
\vspace{-1.2cm}
\begin{IEEEbiography}[{\includegraphics[width=1in,height=1.25in,clip,keepaspectratio]{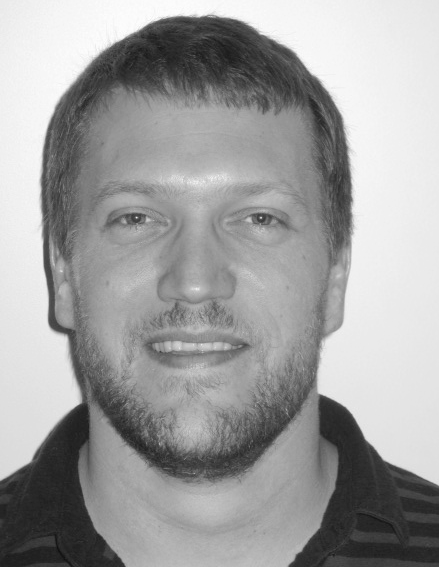}}]{Mathieu
Salzmann}
obtained his MSc and PhD degrees from EPFL in 2004 and 2009,
respectively. He then joined the International Computer Science Institute and the EECS Department at the
University of California at Berkeley as a postdoctoral fellow, and later the
Toyota Technical Institute at Chicago as a research assistant professor. He is now a senior researcher at NICTA in Canberra.
\end{IEEEbiography}
\vspace{-1.2cm}
\begin{IEEEbiography}[{\includegraphics[width=1in,height=1.25in,clip,keepaspectratio]{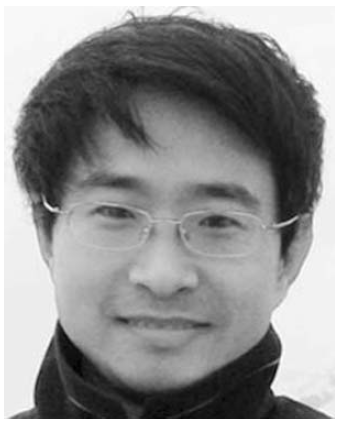}}]{Hongdong Li} is with the computer vision group at ANU. His research interests include 3D Computer Vision, vision geometry, image and pattern recognition, and mathematical optimization. Prior to 2010 he was a Senior Research Scientist with the NICTA, and a Fellow at the RSISE, ANU. He is a recipient of the CVPR Best Paper Award in 2012, Best Student Paper Award at ICPR 2010, CSiRA Best Recognition Paper Prize at DICTA 2013. He was in Program Committees (Area Chair/Reviewer) for recent ICCV, CVPR, and ECCV. He is a member of ARC Centre of Excellence for Robotic Vision, a former member of the ''Bionic Vision Australia''.
\end{IEEEbiography}
\vspace{-1.2cm}
\begin{IEEEbiography}[{\includegraphics[width=1in,height=1.25in,clip,keepaspectratio]{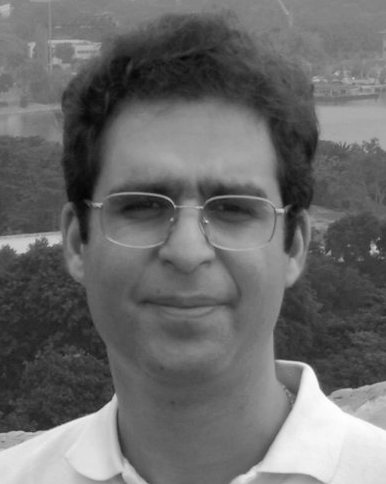}}]{Mehrtash Harandi} received the BSc in Electronics from Sharif University of technology, MSc and PhD degrees in Computer Science from the University of Tehran, Iran. Dr. Harandi is a senior researcher at Computer Vision Research Group (CVRG), NICTA. His main research interests are theoretical and computational methods in computer vision and machine learning with a focus on Riemannian geometry.

\end{IEEEbiography}


\end{document}